\documentclass[sigconf]{acmart}

\usepackage{todonotes}
\usepackage{cleveref}
\usepackage{xspace}
\usepackage{dsfont}
\usepackage{pcatcode}
\usepackage{pgf}
\usepackage[ruled]{algorithm2e}

\newcommand{\mathdefault}[1][]{}

\usepackage{thm-restate}

\theoremstyle{acmplain}

\crefname{algocf}{alg.}{algs.}
\Crefname{algocf}{Algorithm}{Algorithms}

\SetKwInput{kwInit}{Initialization}
\SetKwInput{kwMut}{Mutation}
\SetKwInput{kwSel}{Selection}

\newcommand{\bigO}[1]{\ensuremath{\mathcal{O}\mkern-2mu\left(#1\right)}}
\newcommand{\bigOTilde}[1]{\ensuremath{\tilde{\mathcal{O}}\mkern-2mu\left( #1 \right)}}
\newcommand{\bigOCompact}[1]{\ensuremath{\mathcal{O}\mkern-1mu(#1)}}

\newcommand{\HL}[2]{\textsc{HL}^{#2}\mkern-3mu\left( #1 \right)}
\newcommand{\HD}[2]{\textsc{HD}( #1, #2 )}
\newcommand{\Expected}[1]{\mathbb{E} \left[ #1 \right] }

\newcommand{\timehorizon}{\mathcal{T}}
\newcommand{\timehorizonvalue}{\exp(n^{\varepsilon/8})}
\newcommand{\OneMaxLayer}{\textsc{Ol}}
\newcommand{\goodevent}{\mathcal{G}}

\renewcommand{\Pr}[1]{\text{Pr}\left[ #1 \right]}

\newcommand{\D}{D^{\ell}_{\pm}}
\newcommand{\Dplus}{D^{\ell}_{+}}
\newcommand{\Dminus}{D^{\ell}_{-}}

\newcommand{\ooea}{$(1 + 1)$-EA\xspace}
\newcommand{\olea}{$(1 + \lambda)$-EA\xspace}
\newcommand{\oclea}{$(1 , \lambda)$-EA\xspace}
\newcommand{\moea}{$(\mu + 1)$-EA\xspace}

\newcommand{\OneMax}{\textsc{OneMax}\xspace}
\newcommand{\onemax}{\textsc{OneMax}\xspace}
\newcommand{\ZeroMax}{\textsc{ZeroMax}\xspace}

\newcommand{\OM}{\textsc{Om}\xspace}
\newcommand{\ZM}{\textsc{Zm}\xspace}

\newcommand{\cliff}{\textsc{Cliff}\xspace}

\newcommand{\leadingones}{\textsc{LeadingOnes}\xspace}

\newcommand{\DisOM}{\ensuremath{\textsc{DisOM}}\xspace}
\newcommand{\DisOMpar}{\ensuremath{\textsc{DisOM}_{\mathcal{D}}}\xspace}

\newcommand{\DyDisOmD}{\textsc{DyDisOm}_{\mathcal{D}} }

\newcommand{\eps}{\ensuremath{\varepsilon}}
\DeclareMathOperator*{\argmax}{arg\,max}
\newcommand{\RR}{\ensuremath{\mathbb{R}}}

\newtheorem{remark}{Remark}
\newtheorem{assumption}{Assumption}

\AtBeginDocument{%
  \providecommand\BibTeX{{%
    \normalfont B\kern-0.5em{\scshape i\kern-0.25em b}\kern-0.8em\TeX}}}

\setcopyright{acmcopyright}

\copyrightyear{2024}
\acmYear{2024}
\setcopyright{acmlicensed}\acmConference[GECCO '24]{Genetic and Evolutionary Computation Conference}{July 14--18, 2024}{Melbourne, VIC, Australia}
\acmBooktitle{Genetic and Evolutionary Computation Conference (GECCO '24), July 14--18, 2024, Melbourne, VIC, Australia}
\acmDOI{10.1145/3638529.3654088}
\acmISBN{979-8-4007-0494-9/24/07}

\begin{document}

\title{Plus Strategies are Exponentially Slower for Planted Optima of Random Height}

\author{Johannes Lengler}
\email{johannes.lengler@inf.ethz.ch}
\affiliation{%
  \institution{ETH Zürich}
  \city{Zürich}
  \country{Switzerland}
}

\author{Leon Schiller}
\email{leon.schiller@inf.ethz.ch}
\affiliation{%
  \institution{ETH Zürich}
  \city{Zürich}
  \country{Switzerland}
}

\author{Oliver Sieberling}
\email{osieberling@student.ethz.ch}
\affiliation{%
  \institution{ETH Zürich}
  \city{Zürich}
  \country{Switzerland}
}

\renewcommand{\shortauthors}{Lengler et al.}

\begin{abstract}
We compare the \oclea and the \olea on the recently introduced benchmark \DisOM, which is the \onemax function with randomly planted local optima. Previous work showed that if all local optima have the same relative height, then the plus strategy never loses more than a factor $O(n\log n)$ compared to the comma strategy. Here we show that even small random fluctuations in the heights of the local optima have a devastating effect for the plus strategy and lead to superpolynomial runtimes. On the other hand, due to their ability to escape local optima, comma strategies are unaffected by the height of the local optima and remain efficient. Our results hold for a broad class of possible distortions and show that the plus strategy, but not the comma strategy, is generally deceived by sparse unstructured fluctuations of a smooth landscape.
\end{abstract}

\ccsdesc[500]{Theory of computation~Random search heuristics}

\keywords{Evolutionary Algorithms, Evolutionary Optimization, Randomized Search Heuristics, Run Time Analysis}

\maketitle

\section{Introduction}

A classical topic for Evolutionary Algorithms (EA) are the advantages and disadvantages of elitism, i.e., whether best-so-far solutions should always be kept in the population. Proponents of non-elitist strategies like comma selection argue that the ability to discard the fittest solution may help in escaping local optima. However, theoretical support for this claim has been mixed~\cite{dagstuhl22081theory}, and comma selection does not give benefits in many situations~\cite{dang2021non,doerr22does}. Until recently, the only supportive theoretical pieces of evidence are for \cliff, where comma selection can overcome a large plateau of local optima in polynomial time where plus selection cannot~\cite{hevia2021a}, and for \emph{Distorted Onemax} \DisOM, where a comma strategy outperforms plus strategies in a \onemax landscape with planted local optima~\cite{Jorritsma_Lengler_Sudholt_2023}. 

However, both these results have some limitations. As Jorritsma, Lengler and Sudholt argued in~\cite{Jorritsma_Lengler_Sudholt_2023}, the \cliff function models a rather peculiar type of landscape, and many local optima may not be represented well by this landscape. Concretely, the \cliff function features a massive plateau of local optima, and when an algorithm manages to escape this plateau, then even a random walk ignoring fitness returns to the plateau with high probability. This is arguably different from local optima in many natural problems. It is certainly different from the global optimum (if unique), where a random walk typically brings the algorithm further away from the optimum, not back to it.

For these reasons,~\cite{Jorritsma_Lengler_Sudholt_2023} proposed \DisOM as a new benchmark with a different type of local optima. The benchmark is created from the \onemax function by flipping a coin for each search point, and turning it into a local optimum with some small probability $p$ by adding a distortion $D$ to its fitness.\footnote{This can be seen as frozen Bernoulli noise. Frozen means that two evaluations of the same search point return the same fitness.} The paper~\cite{Jorritsma_Lengler_Sudholt_2023} studied the case of constant $D$ (all distorted points gain the same amount of fitness), and the results were mixed. On the one hand, it was shown that the \oclea can outperform the \olea on \DisOM for some parameter settings by a factor of almost $O(n\log n)$. On the other hand, it was also shown for general parameter settings that the factor can never exceed $O(n\log n)$. The reason was that for constant $D$ the subset of distorted points still looks like a \onemax landscape, just shifted by a constant $D$, and that the plus strategy can still efficiently navigate in that landscape. Thus, in the setting of constant $D$, the gain of comma selection is limited.

\subsection{Our Result, briefly}

In this paper, we also study the \oclea and the \olea on \DisOM, and with the same parameter setting as in~\cite{Jorritsma_Lengler_Sudholt_2023}, see Section~\ref{sec:parameters} for details. But, other than in~\cite{Jorritsma_Lengler_Sudholt_2023}, for each distorted point we choose the size $D$ of the distortion randomly from some distribution $\mathcal D$, where we draw from the same distribution for all distorted points. We make only the very mild assumptions on $\mathcal D$ that it should not decay too extreme, $\Pr{D \ge d}/\Pr{D \ge d+1} \le n^{o(1)}$ for all $d>0$. This includes all distributions with a polynomially or exponentially decaying tail, like exponential distributions or power laws. Our results extend to distributions like the Gaussian distribution, where the condition is only satisfied for small values of $D$. 

We have two main results: firstly, as expected, the local optima do not affect the \oclea, which stays efficient for settings that are efficient on \onemax (see Section~\ref{sec:parameters} for details on the parameter settings). This result is not very surprising and is proved with similar techniques as in~\cite{Jorritsma_Lengler_Sudholt_2023}. The novel part is that the \olea is completely deceived. Of course, it gets stuck in a local optimum, but that is not the worst part. Optimistically, one might hope that it could escape from the local optimum by sampling a search point with a larger \onemax. Although this would need time, it might be tolerable. But before this happens, the algorithm will likely find another distorted point with a slightly larger distortion. Thus it trades its local optimum against another local optimum with a slightly higher peak, which is even harder to escape. This leads to a vicious cycle in which the algorithm entangles itself in increasingly worse local optima. Remarkably, steps towards a higher local optimum typically also bring the algorithm further away from the string $(1\ldots 1)$, so even if the \olea manages to get close to the target region, it will get distracted and lose this achievement.

For many natural distributions $\mathcal D$, including exponential and Gaussian distributions, this leads to superpolynomial time for finding the target fitness. To our best knowledge, this is the first theoretical comparison between plus and comma strategies on ``rugged'' fitness landscapes with local optima of random height. 

\subsection{Related Work}\label{sec:related}
Apart from the paper~\cite{Jorritsma_Lengler_Sudholt_2023} discussed above, there is one other theoretical study of a similar landscape by Friedrich, K\"otzing, Neumann and Radhakrishnan~\cite{FriedrichKNR22}, which they call a $\mathcal D$-rugged \onemax. They compare Random Local Search (RLS) and the \ooea with the compact Genetic Algorithm (cGA) on a \onemax landscape with frozen noise. The main differences in terms of landscape are that they add a distortion to all search points, while we only plant a relatively small number of local optima, and most search points are undistorted. Moreover, they only consider Gaussian distributions (for cGA and RLS) and a geometric distribution (for RLS and the \ooea), while we consider arbitrary distributions with a tail which falls at most exponentially. For the cGA, an estimation-of-distribution algorithm which does not rely on populations, they show that with high probability, the algorithm never samples the same search point twice (in a fixed-target setting with $k^*=\eps n$). Thus, it is irrelevant whether the noise is frozen or not, which allows them to apply previous results about the cGA for non-frozen noise~\cite{friedrich2016compact}. For the \ooea, their results are similar in spirit as our more general result for plus strategies: that the algorithm gets stuck after a few steps.

In another notable paper, Friedrich, K\"otzing, Krejca and Sutton~\cite{friedrich2016graceful} also study different distributions $\mathcal D$, but in the context of non-frozen noise, i.e., two different fitness evaluations of the same search may give two different answers. They consider the \moea and ask what properties of $\mathcal D$ lead to \emph{graceful scaling}, i.e., to the ability of the algorithm of overcoming even large amount of noise for sufficiently large population size. Although the paper considers a rather different setting (non-frozen noise, different algorithms), we suspect that the classification of distributions that they find is similar to the classes of distributions on which plus selection is also fast on \DisOM. Similarly to our result here, they find that exponentially decaying tails of the distribution lead to large runtimes.\footnote{They only consider one particular distribution with a single parameter, but from the exposition they make clear that the exponentially decaying tails are the decisive property of the distribution.} On the other hand, truncated distributions (whose tail suddenly drops to zero) like the uniform distribution on a small finite interval lead to polynomial runtimes, and the proof intuition is similar to those in~\cite{Jorritsma_Lengler_Sudholt_2023} showing that the \olea does not lose more than a factor of $O(n\log n)$ over the \oclea. Mind that the details are quite different in many aspects, but nevertheless it might be worthwhile to explore the similarities of both situations.

For further related work, especially on theoretical work on comma selection, we refer the reader to the discussion in~\cite{Jorritsma_Lengler_Sudholt_2023}.

\section{Setup and Formal Results}

\subsection{Distorted OneMax}
In this paper, we consider optimization w.r.t. the objective function \emph{distorted One Max}, which is a version of $\OneMax$ with randomly planted local optima of variable height. More formally, we start with parameters $p \in [0, 1]$, and $\mathcal{D}$, which is a continuous probability distribution over $\mathbb{R}^+$. We then define our objective function $\DisOMpar: \{0, 1\}^n \rightarrow \mathbb{R}$ as \begin{align*}
  \DisOMpar(x) = \begin{cases}
    |x|_1 & \text{with probability } 1-p\\
    |x|_1 + D & \text{with probability } p.
  \end{cases}
\end{align*} where $D \sim \mathcal{D}$ is a sample from $\mathcal{D}$. We stress that we draw the function \DisOMpar only once before the run of the algorithm (one independent sample for each distorted point). So, if the algorithm queries the same search point twice, it will find the same fitness value. This is different from noisy optimization, where two evaluations of the same search point may yield different fitness values.

We require the following condition on the distribution $\mathcal{D}$, which ensures that the tail of the distribution does not drop too fast. We remark that this assumption is very weak and, in fact, met for almost all commonly used distributions in practice, at least sufficiently close to $0$. The only notable exception are distributions with a bounded support, i.e., such that $\Pr{D \ge d} = 0$ for some $d \in \RR$.  

\begin{assumption}\label{ass:distribution}
Let $\mathcal{D}$ be a continuous distribution and $D \sim \mathcal{D}$. There is some $\sigma = n^{o(1)}$ and some $\hat{d} \in (0, \infty]$ such that for all $d \le \hat{d}$, \begin{align*}
    \Pr{D \ge d } \le \sigma \Pr{D \ge d + 1}. 
\end{align*} Furthermore, $\Pr{D \ge 0} = \Omega(1)$.
\end{assumption}

We assume the distribution $\mathcal{D}$ to be continuous because this allows us to ignore the case that two distorted points have exactly the same distortion, as this does not occur almost surely. Throughout the paper, we will assume that $\mathcal D$ is fixed and does not depend on $n$. We further remark that $\mathcal{D}$ may have a support over negative numbers as long as it is positive with at least constant probability.

\subsection{Main Results}
We study the fixed target performance of the \olea and the \oclea with target $n - k^*$ on \DisOMpar, where $k^* = o(n)$. To this end, we impose some assumptions on the choice of parameters $k^*, \lambda$ and $p$ which will be formally stated in the next section but already referenced in the formal results here. To summarize them very briefly, we have \Cref{ass:oleaineff} which intuitively ensures that $p$ is sufficiently large such that w.h.p., we find a distorted point before reaching the target $n-k^*$ and the \olea will be inefficient, and we have \Cref{ass:params} which is sufficient for the \oclea to be efficient on $\DisOMpar$. More details are found in \Cref{sec:parameters}.

Our first main result states that the \olea is inefficient in this setting if \Cref{ass:distribution} is met and if $p$ is sufficiently large. This result is established by showing that if the \olea ever comes close to the optimum then it reaches a point of distortion at least $d \coloneqq n^{\Omega(1)}$ and $\ZeroMax$-value roughly $n^{1 - \varepsilon}$ with high probability. On the other hand, we show in \Cref{sec:oclea} that the \oclea remains efficient in the same setting.

\paragraph{\textbf{Inefficiency of the \olea on \DisOMpar.}}
Our first main result establishes that the \olea reaches a point of high distortion that is still far away from the optimum.

\begin{restatable}[]{thm}{oleareachesdistortion}\label{thm:oleadistortionisreached}
    Let $\varepsilon > 0$ be a sufficiently small constant. Assume further that $\mathcal{D}$ satisfies \Cref{ass:distribution}, and that $p$ is such that \Cref{ass:oleaineff} is met. Then with high probability, the \olea on \DisOMpar either spends $\timehorizon = \timehorizonvalue$ iterations before reaching a point $x$ with $\ZM(x) \le n^{1 - \varepsilon}/4$, or visits a point $x$ of distortion at least $d \coloneqq \min\{\hat{d}/2,  n^{\varepsilon/16}\}$ before finding the first point with $\ZM(x) < n^{1 - 4\varepsilon}$.
\end{restatable}

Using this, we derive the following general lower bound on the optimization time of the \olea. It shows that the \olea already gets stuck in local optima while it is still far from the optimum, that is, while the number of zeros is still $n^{1 - 4\varepsilon}$ where $\varepsilon$ is an arbitrarily small constant. 

\begin{restatable}[\olea Lower Bound]{thm}{olealowerbound}\label{thm:olealowerbound}
    Consider the \olea on $\DisOMpar$ for a distribution $\mathcal{D}$ satisfying \Cref{ass:distribution}, and $p$ satisfying \Cref{ass:oleaineff}. Fix any sufficiently small constant $\varepsilon > 0$ and let $T$ denote the number of function evaluations until the first point with less than $n^{1- 4\varepsilon}$ zeros is found. Let further $D$ be a random variable with distribution $\mathcal{D}$ and let $d \coloneqq \min\{\hat{d}/2,  n^{\varepsilon/16}\}$. Then, for any function $g(n) = \omega(1)$, we have \begin{align*}
        T \ge \min \left\{ \timehorizonvalue, \frac{1}{g(n) p \Pr{D \ge d}} \right\}
    \end{align*} with high probability.
\end{restatable}

The proof of this theorem essentially relies on \Cref{thm:oleadistortionisreached} which already shows that we either spend $\timehorizonvalue$ steps far from the optimum, or find a point of large distortion, which requires $(g(n) p \Pr{D \ge d})^{-1}$ steps w.h.p. as said number clearly dominates a geometric random variable with expectation $(p \Pr{D \ge d})^{-1}$.

Note that \Cref{thm:olealowerbound} result only gives a lower bound on the time required to reach a point close to the all-ones-string, but not a lower bound on the time required to find the first point of $\DisOMpar$-fitness at least $n- k^*$. However, even in the latter case we can show a similar statement, which we capture in the following remark.

\begin{remark}
    \Cref{thm:olealowerbound} also applies to the time required to find a point of \DisOMpar-fitness $\ge n - k^*$. To see this, note that the above theorem shows that we need time $\min\{\exp(n^{\varepsilon/8}), (p\Pr{D \ge d})^{-1}\}$ before finding the first point with less than $n^{1 - 4\varepsilon}$ zeros w.h.p. In order to reach \DisOMpar-fitness $\ge n - k^*$ during this time, we would need to find a point of distortion at least $d' \coloneqq n^{1-4\varepsilon} - k^*$, which requires roughly $(p\Pr{D \ge d'})^{-1}$ steps, which is already greater than $(p\Pr{D \ge d})^{-1}$ for sufficiently small $\eps >0$.
\end{remark}

We sketch the implications of this general result for some common distributions used in practice. 

\paragraph{\textbf{Exponential Distribution}} Our illustrating example throughout this paper will be the exponential distribution where $\Pr{D \ge d} = \exp(-\varrho d)$ for some constant $\varrho > 0$. Here, we obtain a stretched exponential lower bound for the \olea. We remark that it is not hard to extend this result to the case of Laplacian noise as our proof techniques carry over even if some points have negative distortion. 
\begin{restatable}[]{cor}{expbound}
    Consider the \olea on \DisOMpar under \Cref{ass:oleaineff} with $\mathcal{D}$ being an exponential distribution. Then with high probability, the number of fitness evaluations $T$ before the first point with less than $n^{1 - 4\varepsilon}$ zeros is found satisfies \begin{align*}
        T \ge \exp(n^{\Omega(1)}).
    \end{align*}
    \label{cor:expbound}
\end{restatable}

\paragraph{\textbf{Gaussian Distribution}} In the case where $\mathcal{D}$ is a Gaussian distribution with density \begin{align*}
    \rho(d) = \frac{1}{s \sqrt{2\pi}} \exp\left( - \frac{1}{2} \left( \frac{d}{s} \right)^2 \right), 
\end{align*} where $s$ is the standard variation, we still obtain a super-polynomial lower bound. The reason for this is that we can show that it satisfies \Cref{ass:distribution} as long as $d = o(\log(n))$, so our theorems show that we reach distortion $d \ge \frac{\log(n)}{\log\log(n)}$ w.h.p. and the probability of sampling a such distortion from $\mathcal{D}$ is \begin{align*}
    \Pr{D \ge \frac{\log(n)}{\log\log(n)}} \le \exp\left( -\Omega\left(\frac{\log(n)}{\log\log(n)} \right)^2 \right) = n^{-\omega(1)}.
\end{align*} Hence, we obtain the following corollary.

\begin{restatable}{cor}{gaussbound}
    Consider the \olea on \DisOMpar under \Cref{ass:oleaineff} with $\mathcal{D}$ being a Gaussian distribution. Then with high probability, the number of fitness evaluations $T$ before the first point with less than $n^{1 - 4\varepsilon}$ zeros is found satisfies \begin{align*}
        T \ge n^{\Omega(\log(n) / (\log\log(n)^2))} = n^{\omega(1)}.
    \end{align*}
\end{restatable}

\paragraph{\textbf{Pareto Distribution}} The third example we consider is the case where $\mathcal{D}$ is a Pareto distribution where \begin{align*}
    \Pr{D \ge x} = \left( \frac{x}{x_0} \right)^{1-\tau}
\end{align*} for $x \in [x_0, \infty]$ and $\tau > 2$. That is, the tail decays polynomially with exponent $\tau - 1$. We show that in this case, the run time is at least polynomial with an exponent that grows with $\tau$.

\begin{restatable}{cor}{paretobound}
    Consider the \olea on \DisOMpar under \Cref{ass:oleaineff} with $\mathcal{D}$ being a Pareto distribution. Then, for every sufficiently small $\varepsilon > 0$, there is a constant $c > 0$ such that for every function $g(n) = \omega(1)$, the number of fitness evaluations until the first point with less than $n^{1 - 4\varepsilon}$ zeros is found satisfies \begin{align*}
        T  \ge \frac{n^{c(\tau-1)}}{g(n)p}
    \end{align*} with high probability.
\end{restatable}

We remark that we did not optimize the constants involved in the above results. This is especially true for the maximally reachable distortion in \Cref{thm:oleadistortionisreached}. We leave a further strengthening of our results in this regard open for future work.

\paragraph{\textbf{The \oclea remains efficient}}

On the other hand, we show in \Cref{sec:oclea} that the \oclea remains efficient on \DisOMpar under the same assumptions on the parameters $\lambda, p,$ and $k^*$ as in \cite{Jorritsma_Lengler_Sudholt_2023} (these are formally stated and explained in the next section). The analysis here is largely identical to the one in \cite{Jorritsma_Lengler_Sudholt_2023} with some technical adjustments.
\begin{restatable}[\oclea Upper Bound]{thm}{ocleaupperbound}\label{thm:ocleabound}
    Consider the \oclea on \DisOMpar under \Cref{ass:params} and let $T$ be the number of function evaluations until the first point with fitness $n - k^*$ is reached. Then, with high probability, \begin{align*}
        T = \Theta(n \log(n))
    \end{align*}
\end{restatable}

We confirm our theoretical results with some empirical results on the run time difference between the \olea and \oclea on \DisOMpar which we present in \Cref{sec:experiments}.

\paragraph{\textbf{Proof Techniques}}

The largest part of this work concentrates on proving \Cref{thm:oleadistortionisreached} from which most of our results follow. On a high level, we show that if the \olea ever comes close to the optimum then it is very likely to reach a point $x$ whose distortion is a sufficiently large constant and $\ZM(x) \approx n^{1 - \varepsilon}$. This event ``kickstarts'' our process: since the number of zeros is sub-linear in $n$, it is very likely that we find a point of larger distortion before decreasing the number of ones. The reason for this is that \Cref{ass:distribution} ensures that the CDF of $\mathcal{D}$ decays slower than the probability of sampling an offspring that has more ones than the current parent, i.e., the probability of increasing the number of ones decays faster than the probability of increasing the distortion. This discrepancy is so large that we can easily increase our distortion to $\min\{\hat{d}/2, n^{\Omega(1)}\}$ without even once decreasing it. 

The main difficulty in establishing this result is that we are dealing with \emph{frozen noise}, i.e., we only sample the distortion of a point once, regardless of how many times the algorithm ``sees'' it, and we do not get ``fresh randomness'' at every visited point. This may induce dependencies between the points the algorithm visits and the distribution of distorted points in the neighborhood of a visited point, which make it difficult to establish a formal argument. 
The problem here is that (1) using a union bound over all search points in $\{0, 1\}^n$ is too weak for our purposes, and that (2) the neighborhood of a visited point may depend on the points visited and sampled so far, so we cannot simply apply the union bound to only those points we visit.
We develop a novel technique to circumvent this problem by showing that during the first $\timehorizon = \timehorizonvalue$ iterations, the algorithm never explores a large fraction of the neighborhood of a point sufficiently close to the optimum before sampling said point itself. This allows us to conclude that the neighborhood of a visited point $x$ is with high probability largely unexplored when we first visit $x$, so we may defer uncovering the distortion of most points close to $x$ until we visit $x$ itself. We refer the interested reader to \Cref{sec:exploration} where this is formalized.

\subsection{Parameter Setup}\label{sec:parameters}

We do not consider the time to find the optimum, but the time to find some target fitness $n-k^*$. This is customary in the context of (frozen) distortions or (non-frozen) noise~\cite{prugel2015run,Jorritsma_Lengler_Sudholt_2023,FriedrichKNR22} and has two reasons for us: first, the global optimum may be attained by a distorted point which can only be found by some exhaustive search among many candidates. Second, we want to balance two requirements on $\lambda$: on the one hand, if $\lambda$ is too small then the \oclea is not even able to find the target fitness efficiently on \onemax. On the other hand, if $\lambda$ is too large, then each generation is very likely to create a clone of the parent, and the \oclea loses its ability to escape local optima. These two requirements are easier to balance in a fixed-target setting with goal $n-k^*$, instead of asking for the time to find the global optimum. 

We continue by describing which parameter choices for $\lambda, k^*$, and $p$ are reasonable in our setting. These are essentially the same as in~\cite{Jorritsma_Lengler_Sudholt_2023}.

\paragraph{The population size $\lambda$.} We use the abbreviation $\eta = e / (e - 1)$ such that $\eta^{-1}$ is approximately equal to the probability of not cloning the parent when creating a single offspring. We further define $q \coloneqq \eta^{-\lambda}$ to denote the approximate probability of not cloning the parent when creating $\lambda$ offspring. This is a central quantity in the theory of the \oclea as it determines the probability of escaping a local optimum. Furthermore, it is easy to see that if we choose $\lambda$ too large and thus $q$ too small (that is $\lambda \ge C \log(n)$), then the \oclea mimics the behavior of the \olea for the first $n^{\Omega(1)}$ iterations. We therefore consider the case where $\lambda$ is such that $q \ge n^{-1 + \varepsilon}$, i.e., \begin{align*}
    \lambda \le (1 - \varepsilon)\log_\eta(n)
\end{align*} since this is known to be the regime where the \oclea and the \olea behave differently \cite{Lehre_2011}. 

\paragraph{The target $n - k^*$.} 
%
The reason why we consider optimization to $n - k^*$ instead of $n$ is that the \oclea is inefficient in finding the unique optimum of any target function with unique optimum (such as \OneMax) \cite{Rowe_Sudholt_2014}, but it is efficient in reaching a search point with fitness at least $n - k^*$ as long as \begin{align*}
    \lambda \ge (1 + \varepsilon)\log_{\eta}(n/k^*) 
\end{align*} 
for some constant $\eps >0$, as shown in \cite{Antipov_Doerr_Yang_2019}. In total this means that we assume $(1 + \varepsilon)\log_{\eta}(n/k^*) \le \lambda \le (1 - \varepsilon)\log_\eta(n)$, which is a non-empty range if $k^* = n^{\Omega(1)}$.

\paragraph{The distortion probability $p$} It is not hard to see that $p$ cannot be chosen too small, otherwise (more precisely if $p = o(1 / (n \log(n)))$), the \olea does not sample any distorted point w.h.p. and thus remains fast. To make our results regarding the inefficiency of the \olea applicable, we require a slightly stronger assumption, which is needed to ensure that we reach a point of sufficiently large distortion sufficiently close to the optimum, which is needed to ``kickstart'' our process or reaching points of higher and higher distortion. 

\begin{assumption}[Sufficient for Inefficiency of the \olea]\label{ass:oleaineff}
    Let $\varepsilon > 0$ be any constant. We assume that there is a constant $C > 0$ such that $\lambda \le C \log(n)$, and that for \textbf{all} constants $c > 0$ \begin{align*}
        p \sigma^{-c} = \omega(1 / (n \log(n))),
    \end{align*}
where $\sigma = n^{o(1)}$ is the parameter from~\Cref{ass:distribution}. In particular, this assumption is satisfies when $p\ge n^{-1+\eps}$ for some constant $\eps>0$.
\end{assumption}

We will further make two additional assumptions for the sake of technical simplicity that will facilitate our proof that the \oclea is efficient on \DisOMpar. These are the same assumptions as in~\cite{Jorritsma_Lengler_Sudholt_2023}. First, we assume that $p = o(k^*/n)$ to ensure that -- at the target fitness -- going closer to the optimum is more likely than sampling another distorted point. Secondly, we make the assumption that $q = \omega(p \lambda)$ to ensure that if there is no clone of the current individual, it is likely to not sample a second distorted point. To make matters simpler, we even require the slightly stronger condition $q \ge p^{1 - \varepsilon}$. We further need the condition $q \le (k^*/n)^{1 + \varepsilon}$ to ensure a sufficiently large population size to be efficient for optimizing to target $n- k^*$ on \OneMax. Hence, given a suitable value of $k^*$ and $\lambda$ to ensure that the \oclea is efficient on \OneMax, we additionally require $p$ to be sufficiently small to further ensure efficiency on $\DisOMpar$. We remark that the latter is likely not necessary, and we leave a study of the \oclea without it open to future work. Summarizing, we require the following for our analysis of the \oclea.

\begin{assumption}[Sufficient for Efficiency of the \oclea]\label{ass:params}
    Let $\varepsilon > 0$ be any constant and let $k^* = n^{\Omega(1)}$. We assume that \begin{align*}
        p^{1 - \varepsilon} \le q \le (k^*/n)^{1 + \varepsilon},
    \end{align*} or equivalently \begin{align*}
        (1+\varepsilon)\log_{\eta}(n/ k^*) \le \lambda \le (1-\varepsilon)\log_\eta(1 / p).
    \end{align*}
\end{assumption}

\section{Preliminaries}

\paragraph{General Notation.}
We denote by $[n]$ the set of natural numbers at most $n$, i.e., $[n] = \{1, \ldots, n\}$, and we use standard Landau notation for expressing asymptotic behavior with respect to $n$. Furthermore, we sometimes use the notation $\bigOTilde{g(n)}$ to hide (poly-)logarithmic factors, i.e., $f(n) = \bigOTilde{g(n)}$ if there is a constant $c >0$ such that $f(n) \le \log^c(n) g(n)$ for sufficiently large $n$. All logarithms are with base $e$ if not indicated otherwise. An event holds \emph{with high probability} (w.h.p.) if its probability tends to $1$ for $n\to\infty$. 

We further denote by $\{0, 1\}^n$ the set of all bit strings, and for any $x \in \{0,1\}^n$, $\OM(x), \ZM(x)$ denotes the number of one-bits and zero-bits of $x$, respectively.
For two bit-strings $x, y \in \{0,1\}^n$, we denote by $\HD{x}{y}$ the Hamming-distance of $x$ and $y$, i.e., the number of positions in which $x, y$ differ. We further denote by $\HL{x}{\ell} = \{ y \in \{0,1\}^n \mid \HD{x}{y} = \ell \}$ the set of all bit strings within Hamming-distance $\ell$ of $x$. We call the set $\HL{x}{\ell}$ the \emph{Hamming-layer} $\ell$ of $x$. We further denote by $\HL{x}{\le \ell} = \bigcup_{i=1}^\ell \HL{x}{i}$ the union of all Hamming-layers within Hamming-distance at most $\ell$.  

\paragraph{Algorithms}

We consider the \olea and the \oclea for $\lambda \in \mathbb{N}$. See \Cref{alg:olea} and \Cref{alg:oclea} for pseudocode. Both algorithms aim to maximize a fitness function $f$ and maintain a solution candidate $x_t$. In each iteration, they create $\lambda$ offspring independently using standard mutation that flips each bit independently with probability $1/n$. The only difference is that the \olea updates its solution candidate only if there is an offspring of fitness at least $f(x_t)$ whereas the \oclea always chooses $x_t$ to be the offspring of largest fitness among all offspring generated in the previous iteration (ties are broken uniformly at random in both algorithms).

\SetKwComment{Comment}{}{}

\begin{algorithm}
\caption{\olea for maximizing fitness function $f$ to target $n - k^*$ }\label{alg:olea}
\kwInit{Choose $x_0 \in \{0,1\}^n$ uniformly at random.}
\While{$f(x_t) < n - k^*$}{
    \kwMut{\For{$i = 1$ \KwTo $\lambda$}{
        $y_t^{(i)} \gets \text{mutate}(x_t)$ by flipping each bit in $x_t$ independently with probability $1/n$.
        }
    }
    \kwSel{
        $i_{\text{opt}} \gets \argmax_{i \in [\lambda]} f(y_t^{(i)})$ breaking ties uniformly at random\\
        \hspace{1.45cm} \textbf{if } $f(y_t^{(i_\text{opt})}) \ge f(x_t)$ \textbf{ then} $x_{t+1} \gets y_t^{(i_\text{opt})}$
    }
  }
\end{algorithm}

\begin{algorithm}
\caption{\oclea for maximizing fitness function $f$ to target $n - k^*$}\label{alg:oclea}
\kwInit{Choose $x_0 \in \{0,1\}^n$ uniformly at random.}
\While{$f(x_t) < n - k^*$}{
    \kwMut{\For{$i = 1$ \KwTo $\lambda$}{
        $y_t^{(i)} \gets \text{mutate}(x_t)$ by flipping each bit in $x_t$ independently with probability $1/n$.
        }
    }
    \kwSel{
        $i_{\text{opt}} \gets \argmax_{i \in [\lambda]} f(y_t^{(i)})$ breaking ties uniformly at random.\\
        \hspace{1.5cm}$x_{t+1} \gets y_t^{(i_\text{opt})}$
    }
  }
\end{algorithm}

\paragraph{Auxiliary Lemmas.} We further use the following estimation of the factorial an the binomial coefficient. 

\begin{lemma}\label{lem:fac}
    For every $k \in \mathbb{N}$ we have $k ! = \Theta(k)^k$.
\end{lemma} \begin{proof}
    By Stirling's approximation, there is a constant $c > 0$ such that \begin{align*}
    k! &\le c \sqrt{k} (k/e)^{k} \\
    &= \left( k \cdot c^{1/k}k^{1/(2k)}/e \right)^k\\
    &= \bigO{k}^k
  \end{align*} as $c^{1/k}k^{1/(2k)} = \Theta(1)$ for $k \ge 1$. The fact that $k! = \Omega(k)^k$ can be shown analogously.
\end{proof}

\begin{lemma}\label{lem:binomapprox}
  If $k \le n/2$ then \begin{align*}
    \binom{n}{k} = n^k \Theta(k)^{-k}
  \end{align*}
\end{lemma} \begin{proof}
  It is immediate that \begin{align*}
    \frac{n^k}{k!} \ge \binom{n}{k} \ge \frac{(n-k)^k}{k!}.
  \end{align*} By \Cref{lem:fac}, we have $k! = \Theta(k)^k$ and we immediately conclude that $\binom{n}{k} \le n^k/k! = n^k \Omega(k)^{-k}$. For the other direction, we use that \begin{align*}
    \binom{n}{k} \ge \frac{(n-k)^k}{k!} = \frac{n^k(1-k/n)^k}{k!}.
  \end{align*} Due to our assumption $k \le n/2$, the term $(1-k/n)^k$ is $\Omega(1)^k$ and thus $\binom{n}{k} = n^k \bigO{k}^{-k}$.
\end{proof}

We further need the following Chernoff bound. 
\begin{theorem}[\cite{Dubhashi_Panconesi_2009}]\label{thm:Chernoff}
  Let $X_1, \ldots, X_n$ be i.i.d. random variables with range in $[0, 1]$ and let $X \coloneqq \sum_{i=1}^n X_i$. Then, \begin{enumerate}
      \item[(i)] $\Pr{X \le (1 + \delta) \Expected{X}} \le \exp\left( -\delta^2 \Expected{X} / 3 \right)$ for $0 < \delta < 1$.
      \item[(ii)] $\Pr{X \le (1 - \delta) \Expected{X}} \le \exp\left( -\delta^2 \Expected{X}/2 \right)$ for $0 < \delta < 1$.
      \item[(iii)] $\Pr{ X \ge t } \le 2^{-t}$ for all $t \ge 2e\Expected{X}$.
  \end{enumerate}
\end{theorem}

Moreover, we need the following multiplicative drift theorem with tail bounds which we apply it in a slightly non-standard context in \Cref{lem:onlylogninfitnesslayers}.

\begin{theorem}[Theorem 2.5 in \cite{Lengler_Steger_2018}]\label{thm:multidrift}
    Let $(Z_t)_{t \in \mathbb{N}}$ be a Markov chain with state space $\mathcal{Z}$ and trace function $\alpha: \mathcal{Z} \rightarrow \{0\} \cup [1, \infty)$ with $\alpha(Z_0) = n$, and assume that there is some $\delta > 0$ such that for all $z \in \mathcal{Z}$ \begin{align*}
        \Expected{ \alpha(Z_{t+1}) \mid Z_{t} = z } \ge (1 - \delta) \alpha(z).
    \end{align*} Let further $T = \inf \{ t\ge 0 \mid \alpha(Z_t) = 0\}$. Then for any $k > 0$, \begin{align*}
        \Pr{T \ge \left \lceil \frac{\log(n) + k}{|\log(1 - \delta)|} \right \rceil} \le e^{-k}.
    \end{align*}
\end{theorem}

Additionally, we use the following straightforward estimate of the maximum number of bits flipped during a fixed time horizon.

\begin{lemma}\label{lem:numberofbitflips}
    Within the first $\timehorizon$ iterations of the algorithm, the maximum number of bits flipped in a single mutation is at most $\log_2(\timehorizon) + \log^2{n}$ with probability at least $1 - n^{-\omega(1)}$.
\end{lemma} \begin{proof}
  The number of bits flipped in a single mutation is a binomial random variable $R \sim \text{Bin}(n, 1/n)$. By a Chernoff bound (\Cref{thm:Chernoff}), we have that \begin{align*}
    \Pr{R \ge r} \le 2^{-r}
  \end{align*} for $r \ge 2e$.  Hence, by Markov's inequality, the probability that we flip more than $r$ bits at least once during $\mathcal{T}$ iterations is at most $\timehorizon 2^{-r}$. Setting $r = \log_2 (\timehorizon) + \log^2(n)$ this is at most $n^{-\omega(1)}$ as desired.
\end{proof}

\section{Analysis of the $(1+\lambda)$-EA}\label{sec:1+1}
We give a high level description of why plus selection takes super-polynomial time before giving the formal proofs. The idea is as follows. The $(1+\lambda)$-EA will very likely run into a distorted point while being within distance roughly $n^{1 - \varepsilon}$ to the optimum for any sufficiently small constant $\varepsilon$. If we are at a point of distortion $d$, then in order to accept a search point of lower distortion, we would have to increase the number of ones in the current individual. However, as the number of zeros is sub-linear, this is rather unlikely. Instead, if $\mathcal{D}$ satisfies \Cref{ass:distribution}, it is much more likely to accept a point of distortion at least $d + 1$ because \Cref{ass:distribution} ensures that the CDF of $\mathcal{D}$ decays slower than the probability of increasing the number of ones. Indeed, if we uniformly sample a point in $\HL{x}{\ell}$, then the probability of finding an acceptable individual of distortion at least $d+ 1$ is at least $p^+ \coloneqq p\Pr{D \ge d+\ell}$. On the other hand, the probability of decreasing distortion by finding a point that increases the number of ones is at most $p^- \coloneqq n^{-\Omega(\varepsilon \ell)}p\Pr{D \ge d - \ell}$. Hence, we get that $p^+/p^- \ge n^{\Omega(\varepsilon\ell)}\sigma^{2\ell}$, which is $n^{\Omega(1)}$ for all $\ell \ge 1$. Hence, the probability of not increasing the current distortion by at least $1$ is only $n^{-\Omega(1)}$, which implies that we can reach a distortion of $n^{\Omega(1)}$ w.h.p. without even once going back to a smaller distortion.

However, the formal proof of this is significantly complicated by the fact that the distortion of each search point is assigned only once and does not change if we sample a given point multiple times, so we do not get ``fresh randomness'' at each visited search point and potentially even introduce dependencies between visited points and their neighborhoods. One way to remedy this would be to use concentration bounds to show that the number of points of a certain distortion in the neighborhood of a given point concentrates well around its expectation. However, we would have to do this for all possible distortions we are interested in and then apply union bounds over \emph{all} search points in $\{0,1\}^n$, which are simply too many to make the statements obtained this way meaningful. Another approach would be to apply the union bound only to those individuals we actually visit. However, this does not solve the first issue and -- more importantly -- it does not allow us to get rid of possible dependencies between the set of visited points and their neighborhoods. The problem here is that the number of search points close to $x$ that have a certain distortion might depend on which search points we visited before $x$; indeed if we visit many points close to $x$ before $x$ itself, then we know that all the points visited before $x$ (and the points sampled and discarded before visiting $x$) have fitness at most as large as that of $x$ and hence, the number of points with a certain distortion that are close to $x$ might be very different than what we would expect to the case for the ``typicial'' search point. For this reason, we use the following subsection to develop a technique for largely getting rid of these dependencies. We accomplish this by establishing that for every visited point $x$ sufficiently close to the optimum during the first $\timehorizon = \timehorizonvalue$ iterations, the algorithm has only ``seen'' a small fraction of the neighborhood of $x$ before finding $x$ itself.

\subsection{Ensuring uniform Exploration of the Search Space}\label{sec:exploration}

Our argument consists of two steps. First -- in \Cref{lem:onlylogninfitnesslayers} -- we show that w.h.p. from the set of all points with a certain number of ones, only a small number of points is ever visited. To this end, we define the $k$-th \OneMax-layer as $\OneMaxLayer(k) = \{ x \in \{0, 1\}^n \mid \OM(x) = k \}$ and show that only roughly $\log(|\OneMaxLayer(k)|)$ points in $\OneMaxLayer(k)$ are visited. Intuitively, for non-distorted points, this follows because a fitness improvement is almost as easy to find as a search point of the same fitness. For distorted points, we use the fact that once we found the first distorted point in $\OneMaxLayer(k)$, we can only accept another individual in $\OneMaxLayer(k)$ if this point has a larger distortion than before. However, each time we accept a new individual, we cut the number of individuals with even higher distortion in half (in expectation). See \Cref{fig:halving} for an illustration. In the second part of our argument (\Cref{lem:onlyfewoftheneighborhoodsiuncovered}), we use this fact to show that with high probability -- during the first $\timehorizon$ steps -- we only uncover a sub-constant fraction of $\HL{x}{\ell}$ before sampling $x$ itself for the first time, for all points $x$ which are sufficiently close to the optimum.

\begin{lemma}\label{lem:onlylogninfitnesslayers}
  Let $x \in \{0, 1\}^n$ be a fixed search point and let $k \in \mathbb{N}$. If $\hspace{.1cm}V_{k}$ denotes the set of accepted search points in $\OneMaxLayer(k)$, then with probability at least $1 - n^{-\omega(1)}$,
  \begin{align*}
      |V_{k}(x)| = \bigO{\log(|\OneMaxLayer(k)|) + \log^2(n)}.
  \end{align*}
\end{lemma}\begin{proof}
  Recall that $\OneMaxLayer(k)$ is the set of all bit strings with exactly $k$ ones. We show the statement in two parts. First, we show that only $\log^2(n)$ non-distorted points with $k$ ones are visited. Afterwards, we show that the number of visited points that are distorted is $\bigOCompact{\log(|\OneMaxLayer(k)| + \log^2(n))}$. 
 
  For the first part, we bound the number of visited search point $x$ in $\OneMaxLayer(k)$ which are not distorted. To this end, we show that -- once we visit a non-distorted point in $\OneMaxLayer(k)$ -- it is very likely that we only sample few other points in $\OneMaxLayer(k)$ before finding a point of strictly larger fitness than $x$. This implies that whenever we go back to an individual in $\OneMaxLayer(k)$ at a later time, this point must be distorted as the overall fitness never decreases. To show this, we note two things. First, if the next point we visit has fewer one-bits than the current point, then it must be distorted with fitness at least as large as $x$. Since the distribution $\mathcal{D}$ is continuous, the probability of getting exactly the same fitness is zero, so the offspring has \emph{strictly} larger fitness. Secondly, if we sample a point with strictly more one-bits than $x$, then we also strictly increase the fitness, regardless of whether or not said point is distorted. Hence, it suffices to bound the number of times we create an offspring with parent in $\OneMaxLayer(k)$ that is also in $\OneMaxLayer(k)$ before we sample the first point with more ones than $x$. To this end, note that in order to sample offspring in $\OneMaxLayer(k)$, we need to flip the same number of one- and zero-bits, so the probability of sampling such an offspring is at most \begin{align*} 
    \sum_{\ell=1}^{n} \hspace{.1cm} \binom{k}{\ell}  \binom{n-k}{\ell}\hspace{.1cm} n^{-2\ell} \le \sum_{\ell=1}^{n} \left( \frac{k (n-k)}{n^2} \right)^{\ell} &= \bigO{\frac{k (n-k)}{n^2}} \\ &= \bigO{(n-k)/n}.
  \end{align*} On the other hand, the probability of increasing the number of one-bits by flipping exactly one zero-bit to a one and not changing anything else is at least $(n-k)/n (1 - 1/n)^{n-1} = \Omega((n-k)/n)$. Accordingly, conditional on the event that a given offspring is either in $\OneMaxLayer(k)$ or increases the number of one-bits, the latter case occurs with probability $\Omega(1)$. Hence, by a Chernoff-bound as in \Cref{thm:Chernoff}, after $\log^2(n)$ such samples, we created at least one offspring with more one-bits with probability at least $ 1 - \exp(-\Omega(\log^2(n)) \ge 1 - n^{-\omega(1)}$. We may still create at most $\lambda-1$ offspring of the same generation afterwards, but those are negligible since $\lambda =O(\log n)$.
  
  For the second part, we bound the number of visited distorted points in $\OneMaxLayer(k)$. To this end, we uncover for all bit-strings in $\OneMaxLayer(k)$ if they are distorted or not, but without uncovering their exact distortion. Then, every time we sample one of these distorted points $z$, we only uncover which points have distortion larger and smaller than $z$ but without uncovering the exact distortion. This way -- due to the elitist behavior of the \olea -- we cut the expected number of visitable search points in half every time we sample one of them. To this end, let $R_0$ be the set of distorted individuals obtained in this way. Clearly, $|R_0| \le |\OneMaxLayer(k)|$. Every time we sample an individual $y$ in $R_0$, we uncover only which individuals in $R_0$ have strictly higher and strictly lower distortion than $y$, respectively. Note that there are no points of equal distortion because $\mathcal{D}$ is a continuous distribution. We denote the set of individuals with strictly higher fitness than $y$ by $R_1$ and repeat the same procedure whenever we sample the first individual in $R_1$ to obtain $R_2$, then $R_3$ and so on. Now in every step $t$, the expected size of $R_t$ is $\Expected{R_t} = (|R_{t-1}|-1)/2 \le |R_{t-1}|/2$.  Using multiplicative drift with tail bounds as in \Cref{thm:multidrift}, we get that the number of steps $T$ until $R_T$ consists only of a single individual is $\bigOCompact{\log(|\OneMaxLayer(k)|) + \log^2(n)}$ with probability at least $1  - n^{-\omega(1)}$.
\end{proof}

We now proceed with the second part of our exploration argument and show that for search points relatively close to the optimum, only a small fraction of its neighborhood is already uncovered when we first visit said point.

\begin{figure}
  \includegraphics[width=.45\textwidth]{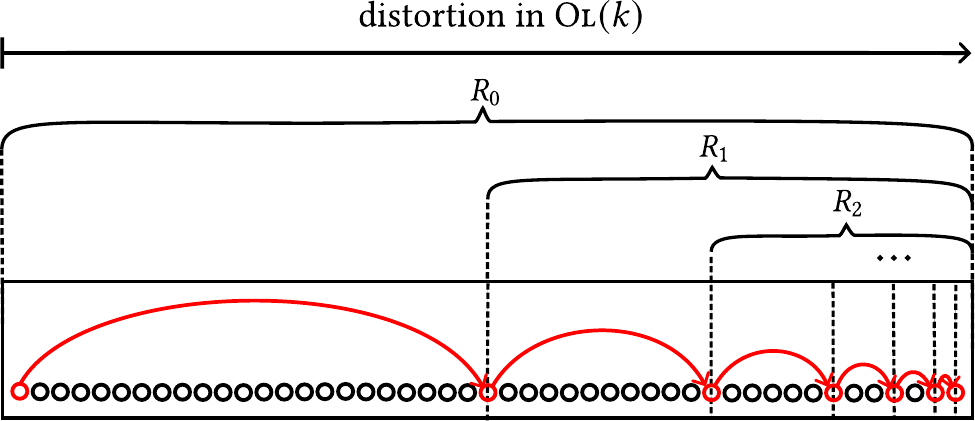}
  \caption{\normalfont Illustration for the proof of \Cref{lem:onlylogninfitnesslayers}. The circles represent all distorted points in $\OneMaxLayer(k)$, the red ones are the ones visited by the algorithm. Every time, we sample a new distorted point $y$ in $\OneMaxLayer(k)$, we cut the set of points with distortion larger than $y$ in half (in expectation). Hence, only $\bigOCompact{\log(|\OneMaxLayer(k)|)}$ distorted points in $\OneMaxLayer(k)$ are visited w.h.p.}\label{fig:halving}
\end{figure}

\begin{figure}
  \includegraphics[width=.45\textwidth]{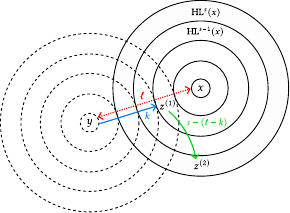}
  \caption{\normalfont Illustration for the proof of \Cref{lem:onlyfewoftheneighborhoodsiuncovered}. We generate offspring from parent $y$ with $\HD{x}{y} = \ell$ in two steps: First decide which $k \sim \text{Bin}(\ell, 1/n)$ bits to flip among the bits that differ between $x$ and $y$ to create the intermediate offspring $z^{(1)}$; then decide which bits to flip among the remaining bits to create the final offspring $z^{(2)}$.}\label{fig:generatingInTwoSteps}
\end{figure}

\begin{lemma}\label{lem:onlyfewoftheneighborhoodsiuncovered}
  Let $\varepsilon, \delta > 0$ be two sufficiently small constants, and set $\timehorizon \coloneqq \exp(n^{\delta})$ and $r \coloneqq n^{2\delta}$. Then with high probability, for all $1 \le s \le r$ and for all points $x \in \{0,1\}^n$ with at most $t \coloneqq n^{1 - \varepsilon} $ zero-bits, only a fraction of \begin{align*}
    \bigO{r^3 (t + r) \log^2(n)/n} + o(1)
  \end{align*} of $\HL{x}{s}$ is uncovered during the first $\timehorizon$ iterations before $x$ is first sampled.
\end{lemma}
\begin{proof}
  Let $y$ be an arbitrary search point that was visited before $x$ is first sampled and let $I$ be the set of bits at which $x$ and $y$ differ. We further define $\ell \coloneqq |I| = \HD{x}{y}$. We call $I$ the \emph{first block} and $[n] \setminus I$ the \emph{second block}, respectively, and we think of the process of creating an offspring with parent $y$ as follows. In the first step, we decide which bits to flip within $I$ and we denote the (intermediate) offspring generated this way by $z^{(1)}$. In the second step, we subsequently decide which bits to flip in $[n]\setminus I$ and denote the final offspring obtained this way by $z^{(2)}$. An illustration for this is given in \Cref{fig:generatingInTwoSteps}. Note that in the first step, the Hamming distance to $x$ decreases while in the second step it increases, i.e., $\HD{x}{z^{(1)}} \le \HD{x}{y}$ and $\HD{x}{z^{(1)}} \le \HD{x}{z^{(2)}}$. To bound the number of search points uncovered in $\HL{x}{s}$ before sampling $x$ itself for the first time, we distinguish two ways of uncovering a search point in Hamming-layer $s$ around $x$. 

  Case 1: $\HD{x}{z^{(1)}} = 0$. Let $x$ be a fixed search point. Assuming case 1, we already flipped all the bits in $I$ after the first sampling step. In the second step, we then have a constant probability of not flipping any other bit and thus sampling $x$ itself. This implies that the number of such samples is bounded. More precisely, among $n/\log(n)$ samples in which $z^{(1)} = x$, the probability that we have at least one sample for in which also $z^{(2)} = x$ is at least \begin{align*}
    1 - \exp( - \Omega(n/\log(n)) ) = 1 - \exp( -n^{1 - o(1)} ).
  \end{align*} Even if all these $n/\log(n)$ samples fall into $\HL{x}{s}$ (for $1 \le s \le r$), this only corresponds to a fraction of $o(1)$ uncovered points as $|\HL{x}{s}| \ge n$ for all relevant $s$. Note that this argument works regardless of how many points we visit before sampling $x$. Furthermore, the number of bit strings with at most $n^{1-\varepsilon}$ zero-bits is at most \begin{align*}
    \sum_{k = 0}^{n^{1-\varepsilon}} \binom{n}{k} \le \sum_{k = 0}^{n^{1-\varepsilon}} n^{k} \le n^{n^{1- \varepsilon}} \sum_{k = 0}^{n^{1-\varepsilon}} n^{-k} = cn^{n^{1-\varepsilon}}
  \end{align*} where $c> 0$ is a constant. By a union bound, we conclude that our statement applies to all such strings $x$ with probability at least \begin{align*}
    &1 - cn^{n^{1-\varepsilon}}\exp\left( -n^{1 - o(1)} \right) \\
    &\hspace{1cm} = 1 - c\exp\left(\log(n) n^{1-\varepsilon} - n^{1-o(1)} \right) = 1 - o(1).
  \end{align*}
    
  Case 2: $\HD{x}{z^{(1)}} > 0$. In this case, we show that (for suitable $r,k$), the number of points in $\HL{x}{s}$ that can be reached from a fixed parent $y$ in this way is sub-constant \emph{deterministically}. Afterwards, we use \Cref{lem:onlylogninfitnesslayers} to bound the number of distinct parents $y$ and \Cref{lem:numberofbitflips} to bound the maximal number of bits flipped in a single iteration to show that all points visited sufficiently close to $x$ combined can only uncover a small fraction of $\HL{x}{s}$. In the following, we assume that $\ell \le 2r$, because we will later show that at most $r$ bits are flipped in a single iteration, so for $\ell$ larger than $2r$, we cannot reach $\HL{x}{s}$.

  For the first part, consider a fixed parent $y$ and recall that $\ell = |I| = \HD{x}{y}$. To reach a point in $\HL{x}{s}$ from $y$, we need to flip at least $k_{\min} = \max\{0, \ell - s\}$ and at most $k_{\max} = \ell - 1$ bits in $I$ as otherwise, we either do not reach $\HL{x}{s}$ or we are in case 1. We count all points in $\HL{x}{s}$ that can be reached from parent $y$  by flipping between $k_{\min}$ and $k_{\max}$ bits in $I$. To this end, let $A_k(x,y)$ be the set of points that differ in $k$ bits from $y$ in $I$. Note further that if we flip $k$ bits in $I$, then we need to additionally flip $s - (\ell - k)$ bits in $[n] \setminus I$ to reach $\HL{x}{s}$. Accordingly, the fraction of points in $\HL{x}{s}$ reachable from $y$ this way is
  \begin{align*}
    \sum_{k = k_{\min}}^{k_{\max}} \frac{|\HL{x}{s} \cap A_k(x,y)|}{|\HL{x}{s}|} & = \sum_{k = k_{\min}}^{k_{\max}} \binom{\ell}{k} \binom{n - \ell}{s - (\ell - k)} / \binom{n}{s}\\
    & \le \sum_{k = k_{\min}}^{k_{\max}} \binom{\ell}{k} \binom{n}{s - (\ell - k)} / \binom{n}{s}\\
    &\le \sum_{k = k_{\min}}^{k_{\max}} \binom{\ell}{k} \frac{(n-s)! s!}{(n-s+\ell - k)!(s - \ell + k)!}.
  \end{align*}
  It is easy to see that $\frac{s!}{(s - \ell + k)!} \le s^{\ell-k}$ and that $\frac{(n-s)!}{(n-s+\ell - k)!} \le (n-s)^{-(\ell - k)}$. Accordingly, \begin{align*}
    \sum_{k = k_{\min}}^{k_{\max}} \frac{|\HL{x}{s} \cap A_k(x,y)|}{|\HL{x}{s}|} &\le \sum_{k = k_{\min}}^{k_{\max}} \binom{\ell}{k} \left( n - s \right)^{k-\ell} s^{\ell - k}\\
    &= \sum_{k = k_{\min}}^{k_{\max}} \binom{\ell}{k} \left( \frac{s}{n-s} \right)^{\ell - k}\\
    &= \sum_{k = k_{\min}}^{k_{\max}} \binom{\ell}{\ell - k} \left( \frac{s}{n-s} \right)^{\ell - k}\\
    &\le \sum_{k = k_{\min}}^{k_{\max}} \bigO{\frac{\ell}{\ell-k}}^{\ell-k} \left( \frac{s}{n-s} \right)^{\ell - k}\\
    &\le \sum_{k = k_{\min}}^{k_{\max}} \bigO{ \frac{\ell s}{(\ell - k) (n - s)} }^{\ell - k}\\
    &\le \sum_{k = k_{\min}}^{k_{\max}} \bigO{ \frac{\ell s}{n - s} }^{\ell - k}.
  \end{align*} Now, because the sum is geometric, because $k \le \ell -1$ and because $\ell s \le 2r^2 = o(n)$ (for sufficiently small $\delta$, i.e., $\delta < 1/4$), we conclude that \begin{align*}
    \sum_{k = k_{\min}}^{k_{\max}} \frac{|\HL{x}{s} \cap A_k(x,y)|}{|\HL{x}{s}|} &= \bigO{ \frac{\ell s}{n} }.
  \end{align*}

  We have thus shown that for every visited point $y$ at Hamming distance $\ell$ of $x$, only a fraction of $\bigO{\ell s / n}$ of $\HL{x}{s}$ is uncovered for all $1 \le s \le r$ under the prerequisites of case 2. However, there might be multiple points $y$ we visit before sampling $x$ for the first time. To bound this number, we use \Cref{lem:onlylogninfitnesslayers}, and we denote by $\mathcal{E}_{\text{visit}}$ the event that the number of visited points with exactly $t$ one-bits is \begin{align*}
    \bigO{ \log(|\OneMaxLayer(t)|) + \log^2(n)} 
  \end{align*} for all $t \in [n]$. By \Cref{lem:onlylogninfitnesslayers} and a union bound, this occurs with probability $1 - n^{-\omega(1)}$. Furthermore, we let $\mathcal{E}_{\text{flip}}$ be the event that the maximum number of bits flipped in a single mutation during the first $\timehorizon = \exp(n^{\delta})$ iterations is $r = n^{2\delta}$. Note that by \Cref{lem:onlylogninfitnesslayers} and \Cref{lem:numberofbitflips}, $\mathcal{E}_{\text{visit}} \cap \mathcal{E}_{\text{flip}}$ occurs with probability $1 - n^{-\omega(1)}$.

  Assuming this, we bound the total fraction of the number of uncovered points in $\HL{x}{s}$ under the prerequisites of case 2. To this end, note that conditional on $\mathcal{E}_{\text{flip}}$, only points within Hamming distance $r + s \le 2r$ of $x$ can uncover points in $\HL{x}{s}$, so we only have to take case $\ell \le 2r$ into account. Note that the set of points within hamming distance $2r$ of $x$ covers at most $4r$ \OneMax-layers. Furthermore, because $x$ has at most $t \coloneqq n^{1-\varepsilon}$ zero-bits, the size of each such level is at most \begin{align*}
    \binom{n}{t + 4r} \le n^{t + 4r},
  \end{align*} and since $\mathcal{E}_{\text{visit}}$ occurs, in each such level only \begin{align*} 
    \bigO{\log(n^{t + 4r}) + \log^2(n)} = \bigO{(t + r) \log^2(n) }
  \end{align*} points are visited. Therefore, if $\mathcal{E}_{\text{visit}} \cap \mathcal{E}_{\text{flip}}$ occurs, then the fraction of uncovered points in $\HL{x}{s}$ under the prerequisites of case 2 is at most \begin{align*}
    \bigO{ r (t + r) \log^2(n) \frac{\ell s}{n}} = \bigO{\frac{r^3( t  + r) \log^2(n)}{n}}
  \end{align*} where we bounded $\ell \le 2r$ and $s\le r$. 
\end{proof}

\subsection{Climbing up the Distortion Ladder}

In this section, we prove our first main theorem stating that under \Cref{ass:distribution} and with high probability, the \olea reaches a point of large distortion before coming close to the optimum. \oleareachesdistortion*

We prove this statement by showing that we likely reach a point that has sufficiently large distortion and \ZeroMax-value in the order of $n^{1 - \varepsilon}$ for arbitrarily small $\varepsilon> 0$. When this happens, it is very likely that the next accepted individual has larger distortion. 

In the following, we will fix the time horizon $\timehorizon = \timehorizonvalue$ as in \Cref{thm:oleadistortionisreached} mainly to ensure a bounded number of bit-flips in a single iteration. We further denote by $\goodevent$ the event described in \Cref{lem:onlyfewoftheneighborhoodsiuncovered} for $\delta = \varepsilon/8$. That is, $\goodevent$ denotes the event that during the first $\timehorizon$ iterations, for all $x \in \{0, 1\}^n$ with at most $t = n^{1 - \varepsilon}$ zeros and for all $1 \le s \le r = n^{\varepsilon/4}$, the fraction of points uncovered in $\HL{x}{s}$ before first sampling $x$ is at most \begin{align*}
    \bigO{r^3( t  + r) \log^2(n)/n} + o(1) = o(1).
\end{align*}
We note that $\goodevent$ occurs with high probability, and we shall make it clear when we condition on $\goodevent$ in the following lemmas. We start with the following auxiliary statement that bounds the probability of not decreasing the number of ones for a search point with $k$ zeros.
\begin{lemma}\label{lem:bitflipfitnessincrease}
  Let $x$ be the current individual with $\ZM(x) = k = o(n)$, and let $\ell \le n/2$. Let further $y$ be an offspring with parent $x$. Then for all $t \in \mathbb{N}_0$, \begin{align*}
    \Pr{\OM(y) \ge \OM(x) + t \mid \ell \text{ bits flip in } x } = \bigO{ k/n }^{\lceil \frac{\ell + t}{2} \rceil} 
  \end{align*}
\end{lemma} \begin{proof}
  To change \OneMax-fitness by at least $t$, at least $\lceil (\ell + t)/2 \rceil$ of the flipped bits need to be among the current $0$-bits. As there are $\binom{n}{\ell}$ possibilities of flipping $\ell$ bits in total, we have \begin{align*}
    &\Pr{\OM(y) \ge \OM(x) + t \mid \ell \text{ bits flip in } x } \\
    & \hspace{3cm} = \sum_{ i = \lceil \frac{\ell + t}{2} \rceil}^\ell \frac{\binom{k}{i}\binom{n - k}{\ell - i}}{\binom{n}{\ell}} \\
    & \hspace{3cm} \le \sum_{ i = \lceil \frac{\ell + t}{2} \rceil}^\ell \frac{k^i}{i!} \frac{\ell! (n - \ell)!}{(\ell - i)!(n - \ell + i)!}\\
    & \hspace{3cm} \le \sum_{ i = \lceil \frac{\ell + t}{2} \rceil}^\ell \frac{k^i}{(n - \ell)^i} \frac{\ell^i}{i!} \\
    & \hspace{3cm} \le \sum_{ i = \lceil \frac{\ell + t}{2} \rceil}^\ell \frac{k^i}{(n/2)^i} \frac{\ell^i}{\Omega(\ell)^i}
  \end{align*} where in the last step we used $i! = \Omega(i)^i$ by \Cref{lem:fac} and $l \le n/2$. Hence, \begin{align*}
    \Pr{\OM(y) \ge \OM(x) + t \mid \ell \text{ bits flip in } x } &= \sum_{ i = \lceil \frac{\ell + t}{2} \rceil}^\ell \bigO{ k/n }^i \\ 
    &= \bigO{ k/n }^{\lceil \frac{\ell + t}{2} \rceil}
  \end{align*} because $k = o(n)$ so the sum is geometric.
\end{proof} 

We further need the following lemma ensuring that the number of points of a certain distortion in $\HL{x}{\ell}$ concentrates well if $\ell$ is of the same order as $d$.

\begin{lemma}\label{lem:distortionconcentration}
    Consider any fixed visited point $x$ during the first $\timehorizon$ iterations with $\ZM(x) \le n^{1-\varepsilon}$. Assume that $\ell, d$ (with $d \le \hat{d}$) are such that there are  constants $\varepsilon_1, \varepsilon_2 > 0$ such that \begin{align*}
        \varepsilon_1 d \le \ell \le r
    \end{align*} and \begin{align*}
        \frac{\log(\ell)}{\log(n)} + \frac{\log(1/p)}{\ell\log(n)} \le 1 - \varepsilon_2.
    \end{align*}
    Let $z$ be a uniform sample from $\HL{x}{\ell}$ and let $Z \subseteq \HL{x}{\ell}$ be the set of points with distortion $\ge d$. Then, we have \begin{align*}
        \frac{|Z|}{|\HL{x}{\ell}|} = \Omega(p \sigma^{-d}) 
    \end{align*} with probability $1 - n^{-\omega(1)}$. 
\end{lemma} \begin{proof}
    Conditional on $\goodevent$, we can uncover a fraction of $(1 - o(1))$ of $\HL{x}{\ell}$ when first visiting $x$. We show that $\Expected{|Z| \mid \goodevent} = n^{\Omega(1)}$, then the statement follows by a Chernoff bound. By \Cref{lem:binomapprox} and \Cref{ass:distribution}, we note that the expected number of points in $Z$ is \begin{align*}
        \Expected{|Z| \mid \goodevent} &\ge (1 - o(1))\binom{n}{\ell} p \Pr{D \ge d} \\ &\ge c_2p\left( \frac{n}{c_1\ell} \right)^{-\ell}  \sigma^{-\ell/\varepsilon_1}\\ &= c_2p\left( \frac{n \sigma^{-1/\varepsilon_1}}{c_1\ell} \right)^\ell = c_2p \cdot n^{\ell(1-o(1))}\ell^{-\ell} 
    \end{align*} where $c_1$ and $c_2$ are constants. To show that this is $\ge n^{\Omega(1)}$, we note that \begin{align*}
        &\frac{\log(\Expected{|Z| \mid \goodevent})}{\log(n)} \\
        &\hspace{1cm}\ge \ell \left( 1-o(1) - \frac{\log(\ell)}{\log(n)} \right) - \frac{\log(1/p)}{\log(n)} + \frac{\log(c_2)}{\log(n)}\\
        &\hspace{1cm}= \ell\left( 1 - \frac{\log(\ell)}{\log(n)} - \frac{\log(1/p)}{\ell\log(n)} - o(1) \right) \pm o(1)\\
        &\hspace{1cm}\ge \ell (\varepsilon_2 - o(1)) \pm o(1)\\
        &\hspace{1cm} = \Omega(1).
    \end{align*} Thus, a Chernoff bound implies that $|Z| = \Theta(\Expected{|Z| \mid \goodevent})$ with probability $1 - n^{-\omega(1)}$ conditional on $\goodevent$. This implies that \begin{align*}
        \frac{|Z|}{|\HL{x}{\ell}|} &= \Theta(\Expected{|Z| \mid \goodevent}) / \binom{n}{\ell}\\
        &= \Theta(p \Pr{D \ge d}) = \Omega(p\sigma^{-d}).
    \end{align*} 
    Now, because $\goodevent$ occurs with probability $1 - n^{-\omega(1)}$, this also holds with probability $1 - n^{-\omega(1)}$ unconditionally, as desired.
\end{proof}

Using this, we show that we likely reach an individual of sufficiently large distortion and roughly $n^{1 - \varepsilon}$ zeros. This will be needed to kickstart our process.

\begin{lemma}\label{lem:wereachdistortion}
  Let $d_0, \varepsilon > 0$ be constants. Then with high probability, the \olea either reaches a point $x$ with distortion at least $d_0$ and $\ZM(x) \in [n^{1-3\varepsilon}, n^{1-\varepsilon}/2]$ within the first $\timehorizon = \timehorizonvalue$ iterations, or $\timehorizon$ iterations pass during which the current individual $x$ fulfills $\ZM(x) \ge n^{1 - \varepsilon}/4$.
\end{lemma} \begin{proof}
  We follow a strategy very similar to \cite[Section 5]{Jorritsma_Lengler_Sudholt_2023}, i.e., we show our statement in the following three parts. The difference is that we have to ensure that we reach a point of \emph{sufficiently large} distortion instead of just any distorted point. As key steps, we shows that the following three points hold w.h.p.
  \begin{enumerate}
    \item[(i)] Either $\mathcal{T}$ iterations pass and all visited points have $\ZM$-value at least $n^{1 - \varepsilon}/4$ or a point $x$ with $\ZM(x) \in I \coloneqq [n^{1  - 2\varepsilon}, n^{1 - \varepsilon}/4]$ is visited.
    \item[(ii)] If a point with $\ZM(x) \in I$ is visited, then during the next $\Omega(n \log(n))$ iterations, the current individual $x$ fulfills $\ZM(x) \ge n^{1 - 3\varepsilon}$.
    \item[(iii)] During these $\Omega(n \log(n))$ iterations, only a sub-constant fraction of iterations produces an offspring of higher fitness and smaller $\ZM$-value. In all iterations where this does not occur, we have a good probability of sampling a point of distortion at least $d_0$ in the 2-neighborhood of $x$.
  \end{enumerate}
  
  We start with part (i). To this end, note that by a Chernoff bound, we likely start with an individual $x$ such that $\ZM(x) = \Theta(n)$. Furthermore, by \Cref{lem:numberofbitflips} the maximum number of bit flips that occur within the first $\timehorizon$ iterations is at most $\bigOCompact{\log(\timehorizon) + \log^2(n)} = \bigOCompact{n^{\varepsilon/8}}$ w.h.p. Hence, we do not jump over the interval $I$ as otherwise, this would imply that we flip $\Theta(n^{1 - \varepsilon})$ bits in a single iteration. 

  For part (ii), we show that once a point in $I$ is visited, the algorithm spends $\mathcal{T}_1 \coloneqq \Theta(n \log(n))$ iterations (with a suitably small hidden constant) during which the current individual $x$ is such that $\ZM(x) \ge n^{1-3\varepsilon}$. To this end, we apply \cite[Theorem 3.6]{Jorritsma_Lengler_Sudholt_2023} stating that our algorithm on any fitness function takes time $\Omega(n \log(a / b))$ to reach a point with Hamming-distance $b$ from a fixed point $x^*$ when starting at a point $x$ with Hamming distance $a$ from $x^*$. In our case, this means that we spend $\Omega(n \log(n))$ many iterations after visiting the first point in $I$ before the first point with less than $n^{1 - 3\varepsilon}$ zeros is visited.

  For part (iii), we show that we reach a point $x$ with distortion at least $d_0$ and $\ZM(x) \in [ n^{1 - 3\varepsilon}, n^{1 - \varepsilon}/2]$ during these $\Omega(n \log(n))$ iterations w.h.p. To this end, we start by considering the first $\log(n)$ iterations after first visiting a point $x$ with $x \in I$, and we show that we escape potential strongly negative distortions within this time. To see this, first note that during this time, the maximum number of bits flipped is $\bigOTilde{1}$, so the number of zeros remains $\le n^{1-\varepsilon}/3$. Now, while the current individual $x$ has distortion $d < -1$ the number of points in $\HL{x}{1}$ that are not distorted or distorted with positive distortion is $\Omega(n)$ w.h.p. To see this, note that the probability that a point is either non-distorted or has distortion $\ge 0$ is $\Omega(1)$, so if $\goodevent$ happens, we get from a Chernoff bound that the number of such points in $\HL{x}{1}$ is $\Omega(n)$ w.h.p. As $\goodevent$ happens w.h.p. as well, the statement follows. Hence, we will accept a point of distortion $d \ge 0$ if we sample one such point while all other generated offspring decrease the number of ones. By \Cref{lem:bitflipfitnessincrease} and since $\lambda = \bigO{\log(n)}$ this occurs with probability at least 
  \begin{align}\label{eq:lower}
    \left( 1 - cn^{-\varepsilon/2} \right)^{\lambda} = \Omega(1).
  \end{align}
  Accordingly, all this happens at least once during $\log(n)$ iterations w.h.p.  

  With this, we can now assume that we start at a point $x$ with distortion $d \ge -1$ (this includes the case where $x$ is not distorted) and $\ZM(x) \le n^{1-\varepsilon}/3$. We now consider the remaining $\mathcal{T}_1 - \log(n) = \Omega(n \log(n))$ iterations, and show that we reach a point of distortion $\ge d_0$ w.h.p. during this time.
  
  To this end, we let $x_1, \ldots, x_m$ be the search points visited in this time interval before the first point with distortion $d_0$ is reached. Again, by \Cref{lem:numberofbitflips}, we note that we only flip at most $\bigOTilde{1}$ bits in a single mutation during this time. This implies that we can assume that $\ZM(x_i) \le n^{1 - \varepsilon}/2$ for all $x_1, \ldots, x_m$ because we start at a point $x$ with $\ZM(x) \le n^{1 - \varepsilon}/3$ as well as distortion $d \ge -1$ and every time we increase the number of zeros, this must be compensated by an increase in distortion that is at least as large, so -- since the maximum number of bits flipped is $\bigOTilde{1}$ -- before falling back to a point with $ \ge n^{1 - \varepsilon}/2$ zeros, we reach a point of distortion at least $d_0$.

  With that in mind, we show that in every iteration, we have a probability of $\omega(1 / (n \log(n)))$ of sampling and accepting a point of distortion $\ge d_0$ while visiting $x_1, \ldots, x_m$. To this end, we apply \Cref{lem:distortionconcentration} and a union bound over all the $m = \bigO{n \log(n)}$ visited search points $x_1, \ldots, x_m$ to conclude that -- while the current individual is any of $x_1, \ldots, x_m$ -- the probability of sampling a point with distortion $d_0 + 2$ in $\HL{x}{2}$ is \begin{align*}
      \Omega(p \sigma^{-(d_0 + 2)}) = \omega(1/(n \log(n)))
  \end{align*} by \Cref{ass:oleaineff}. Note that the prerequisites of \Cref{lem:distortionconcentration} are met because that $p = \omega(1/(n \log(n)))$, $\ell = 2$ and because $d_0$ is constant. So in each iteration, we have a chance of $\omega(1/(n \log(n)))$ that the first sample has distortion $\ge d_0 + 2$ and is in $\HL{x}{2}$. Furthermore -- as noted before in \Cref{eq:lower} -- the probability that all the other $\lambda - 1$ offspring we create have \OneMax-value smaller than $x$ is at least $\Omega(1)$.
  Thus, in each iteration, there is a probability of $\omega(1/(n \log(n)))$ of finding and accepting a point of distortion $\ge d_0$, which implies that this happens at least once during $\Omega(n \log(n))$ iterations w.h.p.
\end{proof}

We proceed by showing that if we visit a point $x$ with distortion $d$ and $\ZM(x) \le n^{1-\varepsilon}$, then -- among the points of higher fitness than $x$ -- the number of points of distortion at least $d + 1$ is (in expectation) much larger than the number of points that have  distortion $< d + 1$. To this end, we estimate the number of points with fitness at least $f(x)$ and distortion larger and smaller than $d+1$, respectively, in $\HL{x}{\ell}$. To this end, we denote by $\D$ the set of points in $\HL{x}{\ell}$ with fitness at least $f(x)$. Let further $\Dminus \subseteq \D$ be the set of search points in $\D$ of distortion smaller than $d + 1$, and let $\Dplus$ be the set of search points with distortion at least $d + 1$. 

\begin{lemma}\label{lem:expectedneighbors}
  Let $x$ be a point visited during the first $\timehorizon$ iterations with $k \le n^{1- \varepsilon}$ zeros and distortion $d > 0$. Let further $\hat{\ell} \coloneqq \min\{r, \hat{d} - d\}$. Then conditional on $\goodevent$, for all $\ell$ with $1 \le \ell \le \hat{\ell}$, we have\footnote{Technically, the statement refers to the conditional expectations $\Expected{|\Dminus| \mid \goodevent}, \Expected{|\Dplus| \mid \goodevent}$, but we omit this notation for the sake of better readability} \begin{align*}
    \Expected{|\Dminus|}/\Expected{|\Dplus|} \le \begin{cases}
      \bigO{k \sigma^4 /n}^{\ell/2} & \text{if } \ell < d\\
      \bigO{k \sigma^4 /n}^{\ell/2} / p & \text{otherwise.}
    \end{cases}
  \end{align*}
\end{lemma} \begin{proof}
  We start by calculating the expectation of $\Dminus$. To this end, note that in order to not decrease fitness but decrease distortion, the number of ones in any $y \in \Dminus$ must be at least as large as $|x|_1$. By \Cref{lem:bitflipfitnessincrease}, the fraction of individuals in $\HL{x}{\ell}$ having this property is $\bigO{k/n}^{\ell / 2}$. Moreover, the number of ones of any $y \in \HL{x}{\ell}$ increases by at most $\ell$ as compared to $x$, so in order for $y$ to have higher fitness, its distortion needs to be at least $\min\{0, d - \ell\}$. Hence, if $D$ denotes a random variable with distribution $\mathcal{D}$, we have that \begin{align*}
    \Expected{|\Dminus|} \le \begin{cases} \binom{n}{\ell} \bigO{k/n}^{\ell / 2} p \Pr{D \ge d - \ell } & \text{if } \ell < d \\
      \binom{n}{\ell} \bigO{k/n}^{\ell / 2} & \text{otherwise.}
    \end{cases}
  \end{align*}

  To bound the expectation of $|\Dplus|$, we use the fact that we condition on $\goodevent$, so when $x$ is first visited, the algorithm has only seen a sub-constant fraction of $\HL{x}{\ell}$ before. We can now uncover all the not-yet-uncovered points and not that with probability at least $p \Pr{D \ge d + \ell}$, a such point has distortion at least $d + \ell$ and at the same time fitness at least $f(x)$ because the number of ones is at least $\OM(x) - \ell$. This shows that \begin{align*}
    \Expected{|\Dplus|} \ge (1 - o(1))\binom{n}{\ell} p \Pr{ D \ge d + \ell}.
  \end{align*} Now, recall that $\Pr{D \ge a} / \Pr{D \ge a + b} \le \sigma^{b}$ if $a + b \le \hat{d}$ by \Cref{ass:distribution}. Furthermore, said assumption yields a general lower bound of $\Pr{D \ge a} = \Omega(\sigma^{-a})$ for all $a \in [0, \hat{d}]$. Hence, dividing our estimates for $\Expected{|\Dminus|}$ and $\Expected{|\Dplus|}$ yields
  \begin{align*}
    \Expected{|\Dminus|}/\Expected{|\Dplus|} &\le \begin{cases}
      \bigO{k/n}^{\ell/2} \sigma^{2\ell} & \text{if } \ell < d\\
      \bigO{k/n}^{\ell/2} \sigma^{d + \ell} / p & \text{otherwise.}
    \end{cases} \\
    &\le \begin{cases}
      \bigO{k/n}^{\ell/2} \sigma^{2\ell} & \text{if } \ell < d\\
      \bigO{k/n}^{\ell/2} \sigma^{2\ell} / p & \text{otherwise.}
    \end{cases}
  \end{align*}
\end{proof}

We the above lemma shows that the number of acceptable search points in $\HL{x}{\ell}$ for a point $x$ with $\ZM(x) \le n^{1 - \varepsilon}$ is larger than the number of acceptable points with smaller distortion by a factor of $n^{\Omega(1)}$. We show that a similar statement also applies with high probability if we instead consider the expectation of the ratio $|\Dminus | / | \D |$, which is equal to the probability of sampling a point of distortion at most $d+1$ when uniformly sampling among the points of fitness $\ge f(x)$ in $\HL{x}{\ell}$ if we factor in the randomness both from the sampling and the distribution of the distorted points.

\begin{lemma}\label{lem:probabilityofincrease}
    Let $x$ be any visited search point during the first $\timehorizon$ iterations with $k = o(n)$ zero-bits and distortion $d > 0$. Then, for any $1 \le \ell \le \hat{\ell} \coloneqq \min\{r, \hat{d} -d\}$ and conditional on $\goodevent$, we have \footnote{Again, this technically applies to $\Expected{|\Dminus | / |\D| \mid \goodevent}$, which we avoid for the sake of simplicity.} \begin{align*}
        \Expected{|\Dminus | / |\D| } = \begin{cases}
      \bigO{k \sigma^4 /n}^{\ell/4}& \text{if } \ell < d\\
      \bigO{k \sigma^4/n}^{\ell/4} / p & \text{otherwise.}
    \end{cases}
    \end{align*}
\end{lemma}\begin{proof}
  We use \Cref{lem:expectedneighbors} and distinguish two cases.
  
  Case 1: $\Expected{|\Dplus|} \le (n/(k\sigma^4))^{\ell/4}$. Assume for now that $\ell < d$. By \Cref{lem:expectedneighbors}, we have \begin{align*}
    \Expected{|\Dminus|} \le \bigO{\frac{k \sigma^4} {n}}^{\ell/2} \left( \frac{k\sigma^4}{n} \right)^{-\ell/4} = \bigO{\frac{k \sigma^4}{n}}^{\ell/4},
  \end{align*} so Markov's inequality tells us that \begin{align*}
    \Expected{|\Dminus | / |\D| } \le \Pr{|\Dminus| > 0} = \bigO{\frac{k \sigma^4}{n}}^{\ell/4}.
  \end{align*} The case $\ell \ge d$ works analogously.

  Case 2: $\Expected{|\Dplus|} > (n/(k\sigma^4))^{\ell/4}$.Again, we assume first that $\ell < d$. Here, we can use a Chernoff bound to conclude that \begin{align*}
   \Pr{|\Dplus| \le \Expected{|\D|} / 2 } \le \exp\left( -\Omega\left(\frac{n}{k\sigma^4}\right)^{\ell/4} \right).
  \end{align*}
  At the same time, we conclude using Markov's inequality that \begin{align*}
    \Pr{|\Dminus| \ge \left(\frac{k\sigma^4}{n}\right)^{\ell/4} \Expected{|\Dplus|} } &\le \left(\frac{n}{k\sigma^4}\right)^{\ell/4} \frac{\Expected{|\Dminus|}}{\Expected{|\Dplus|}} \\
    &= \bigO{\frac{k \sigma^4}{n}}^{\ell/4}. 
  \end{align*} So in total, \begin{align*}
      \Expected{|\Dminus | / |\D| } &\le 2\left(\frac{k\sigma^4}{n}\right)^{\ell/4} + \bigO{\frac{k \sigma^4}{n}}^{\ell/4} \\
      &\hspace{1cm}+ \exp\left( -\Omega\left(\frac{n}{k\sigma^4}\right)^{\ell/4} \right)\\
      &= \bigO{\frac{k \sigma^4}{n}}^{\ell/4}.
  \end{align*} Again, the case $\ell \ge d$ works analogously.
\end{proof}

We finally use this statement to show that the next visited point has distortion at least $d + 1$ with probability $1 - n^{-\Omega(1)}$ and is further within distance $\bigOTilde{d}$ of $x$. The latter ensures that the number of zeros does not drastically change so that we can repeatedly apply this statement afterwards.

\begin{lemma}\label{lem:probincrease}
  Let $\varepsilon > 0$, and let $x$ be a search point visited during the first $\timehorizon$ iterations with $k \le n^{1 - \varepsilon}$ zero-bits and distortion $d$ such that $\max\{2, 12/\varepsilon\} < d \le \min\{ \hat{d}/2, n^{\varepsilon/16} \}$. Let $y$ be the first point accepted after $x$. Let $\mathcal{E}$ be the event that the following two statements hold:
  \begin{enumerate}
    \item[(i)] $y$ has distortion at least $d + 1$
    \item[(ii)] $\HD{x}{y} \le d\log^2(n)$. 
  \end{enumerate}
  Then, $\mathcal{E}$ occurs with probability at least $ 1 - n^{-\varepsilon/8}$.
\end{lemma}\begin{proof}
    Think of our process as follows. When we first visit $x$, we uncover all the not-yet-uncovered points in $\HL{x}{\ell}$ for $1 \le \ell \le \lfloor d \rfloor$. If the next accepted point $y$ is in $\HL{x}{\ell}$, then the probability that said point has distortion $< d+1$ is exactly $|\Dminus|/|\D|$. Accounting for the fact that $|\Dminus|/|\D|$ is again a random variable, said probability (conditional on $\goodevent$) is simply $\Expected{|\Dminus|/|\D|}$, so we can use the previous lemma. Otherwise -- i.e. if $\HD{x}{y} > \lfloor d \rfloor$ -- a different argument is required because \Cref{ass:distribution} is only met up to a maximal distortion of $\hat{d}$ and the previous lemma relies on the fact that there are sufficiently many points of distortion $d + \ell$ in $\HL{x}{\ell}$, which breaks down if $\ell > d$ as $d$ may be as large as $\hat{d}/2$. 
    
    We remedy this by showing that the number of samples until we find a point in $\HL{x}{\lfloor d \rfloor}$ of fitness $> f(n)$ is at most \begin{align*}
        \mathcal{T}_1 \coloneqq \log(n) (c_1\sigma^{2} d)^d/p
    \end{align*} with high probability where $c_1 > 0$ is a constant to be fixed later. This implies that we do not have enough time to find a point $z$ which has $\HD{x}{z} > \lfloor d \rfloor$ and $\OM(z) \ge \OM(x)$, so if the next accepted point is not in $\HL{x}{\le \lfloor d \rfloor}$, then it has distortion $\ge d+1$. The same event further implies that we do not flip too many bits before accepting a new point, so it will also yield part (ii) of the statement.

    We let $\mathcal{E}_1$ be event that we sample a point of distortion $> 2d$ in $\HL{x}{\lfloor d \rfloor}$ within $\mathcal{T}_1$ samples (where the constant $c_1$ will be fixed later), and we show that $\mathcal{E}_1$ occurs with probability $1 - 1/n$. To this end, we apply \Cref{lem:distortionconcentration} to with $\ell = \lfloor d \rfloor$ and disortion $2d$. Note that with this choice -- due to the assumption that $p$ -- we have, \begin{align*}
        \frac{\log(\ell)}{\log(n)} + \frac{\log(1/p)}{\ell\log(n)} \le \frac{\varepsilon}{16} + \frac{1 + o(1)}{2} \le 1 - \varepsilon,
    \end{align*} for sufficiently small $\varepsilon$, so we can indeed apply \Cref{lem:distortionconcentration} to conclude that -- with probability $1 - n^{-\omega(1)}$ -- the neighborhood of $x$ is such that if we sample uniformly from $\HL{x}{\ell}$, we find a point of distortion $> 2 d$ with probability at least $c_2 p \sigma^{-2d}$, where $c_2$ is a positive constant. Let $\mathcal{B}_1$ denote the ``bad'' event that this does not occur\footnote{That is, $\mathcal{B}_1$ is the event that the $\HL{x}{\ell}$ is such that the fraction of points with distortion $\ge 2d$ is smaller than $c_2 p \sigma^{-2d})$} and note that $\Pr{\mathcal{B}_1} \le n^{-\omega(1)}$.

    Note further that the probability of flipping exactly $\ell = \lfloor d \rfloor$ bits is at least \begin{align*}
        \binom{n}{\ell} \hspace{.1cm} \frac{1}{n^{\ell}} \left( 1 - \frac{1}{n} \right)^{n} \ge \frac{n^{\ell}\bigO{\ell}^{-\ell}}{2e n^\ell } = \bigO{d}^{-d}
    \end{align*} by \Cref{lem:binomapprox}. Hence -- conditional on $(\neg \mathcal{B}_1)$ -- the total probability of sampling a point of distortion $> 2d$ in $\HL{x}{\ell}$ is at least \begin{align*}
        p_1 \coloneqq p (c_3\sigma^2d)^{-d}
    \end{align*} for some constant $c_3 > 0$. The probability that this does not occur within $2 \log(n)/p_1 \le \mathcal{T}_1$ (for a suitable choice of $c_1$) samples is at most \begin{align*}
        \left( 1 - p_1 \right)^{2\log(n) /p_1} \le e^{-2\log(n)} = 1/n^2
    \end{align*} Accordingly, \begin{align*}
        \Pr{\neg \mathcal{E}_1} \le 1/n^2 + \Pr{\mathcal{B}_1} \le 1/n.
    \end{align*}

    Now, we show the implications of the above, and we start by showing that $\mathcal{T}_1$ samples are not enough to find a point with more ones than $x$ at Hamming-distance $\ge \lceil d \rceil$.
    To this end, denote by $\mathcal{E}_2$ the event that during $\mathcal{T}_1$ samples, we find a point $z$ with $\HD{x}{z} \ge d $ and $\OM(z) \ge \OM(x)$. Since $\ZM(x) \le n^{1-\varepsilon}$ and due to \Cref{lem:bitflipfitnessincrease} this happens in a single sample only with probability $(c_2n^{-\varepsilon})^{d/2}$ (for some constant $c_2$), so by Markov's inequality, \begin{align*}
        \Pr{\neg\mathcal{E}_2} &\le \log(n) p^{-1} (c_1 \sigma^{2} d)^{d}(c_2n^{-\varepsilon})^{d/2} \\
         &\le \log(n) p^{-1} n^{d(-\varepsilon/2 + \varepsilon/16 + o(1))} \\
         &\le \log(n) p^{-1} n^{d(-\varepsilon/2 + \varepsilon/16 + o(1))}\\
         & \le \log(n) p^{-1} n^{-d\varepsilon/4}.
    \end{align*} Due to our assumption that $d \ge 12/\varepsilon$ and $p = \omega(1/(n\log(n)))$, we get \begin{align*}
        \Pr{\neg\mathcal{E}_2} \le \log^2(n) n n^{-3}\le 1/n.
    \end{align*}

    Finally, we show that the maximum number of bits flipped during $\mathcal{T}_1$ samples is bounded. To this end, denote by $\mathcal{E}_3$ the event that more than $d\log^2(n)$ bits are flipped during  $\mathcal{T}_1$ samples. By \Cref{lem:numberofbitflips}, we get that with probability $1 - n^{-\omega(1)}$, the maximum number of bits flipped during this time is at most \begin{align*}
        &d \log_2( c_1 \sigma^2 d ) + \log_2(\log(n)) + \log_2(1/p) + \log^2(n) \\
        &\hspace{1cm}\le \bigOCompact{d \log(n) + \log\log(n) + \log(n)} + \log^2(n)\\
        &\hspace{1cm}\le d\log^2(n)
    \end{align*} because $d \ge 2$ and because we may assume $n$ to be large enough.
    In total, this yields  \begin{align*}
        \Pr{\neg(\mathcal{E}_1 \cap \mathcal{E}_2 \cap \mathcal{E}_3)} &\le n^{-1} + n^{-1} + n^{-\omega(1)}\\
        &\le 2n^{-1}.
    \end{align*} 
    With this, we bound the probability that the next accepted point has distortion at least $d + 1$ as follows. If $\mathcal{E}_1 \cap \mathcal{E}_2 \cap \mathcal{E}_3$ occurs, then the next point we accept can only have distortion smaller than $d +1$ if it is in $\HL{x}{\le \lfloor d \rfloor}$. In this regime, \Cref{lem:probabilityofincrease} bounds the probability of sampling a point of distortion $< d+1$ when sampling uniformly among the points of fitness $\ge f(x)$ in $\HL{x}{\ell}$ (i.e. sampling in the set $\D$) when factoring in the randomness from both the sampling and the distribution of distorted points in $\D$ (conditional on $\goodevent$). To bound the final probability that $y$ has distortion $< d+1$ if $\HD{x}{y} \le \lfloor d \rfloor$, we produce one sample $y_\ell$ in $\D$ for all $1 \le \ell \le \lfloor d \rfloor$ and denote by $\mathcal{B}$ the ``bad'' event that at least one of them has distortion $< d + 1$. We have
    \begin{align*}
        \Pr{\mathcal{B} \mid \goodevent} &= \Pr{\exists \ell: y_\ell \in \Dminus \mid \goodevent}\\ 
        &\le \sum_{\ell = 1}^{\lfloor d \rfloor} \Expected{|\Dminus|/|\D| \mid \hspace{.05cm} \goodevent} \\
        &\le \sum_{\ell = 1}^{\infty} \bigO{ \frac{k \sigma^4}{n} }^{\ell/4} \le n^{-\varepsilon/4 + o(1)}
    \end{align*} because $k \le n^{1-\varepsilon}$ and because the sum is geometric. In total, we have
    \begin{align*}
        \Pr{\neg\mathcal{E}} &\le \Pr{\mathcal{B} \mid \goodevent} + \Pr{\neg \goodevent} + \Pr{\neg(\mathcal{E}_1 \cap \mathcal{E}_2 \cap \mathcal{E}_3)}\\
        &\le n^{-\varepsilon/4 + o(1)} + n^{-\omega(1)} + 2n^{-1}\\ &\le n^{-\varepsilon/8}
    \end{align*}as desired.
\end{proof}

Plugging everything together, we can prove \Cref{thm:oleadistortionisreached}.

\begin{proof}[Proof of \Cref{thm:oleadistortionisreached}]
  By \Cref{lem:wereachdistortion}, we either reach a point $x$ of distortion at least $d_0 \coloneqq \max\{2, 12/\varepsilon\}$ and $\ZM(x) \in I \coloneqq [n^{1 - 3\varepsilon}, n^{1 - \varepsilon}/2]$, or $\timehorizon$ iterations pass without reaching a point $x$ with $\ZM(x) \le n^{1 - \varepsilon}/4$ w.h.p. In the latter case, our statement follows, so for the rest of this proof, we assume that our process starts at a point $x_0$ with distortion at least $d_0$ and $\ZM(x_0) \in [n^{1-3\varepsilon}, n^{1 - \varepsilon}/2]$. 
  
  Now, \Cref{lem:probincrease} tells us that -- conditional on $\goodevent$ -- if we visit a point $x$ during the first $\timehorizon$ iterations that has distortion $d \ge d_0$ and $\ZM(x) \le n^{1- \varepsilon}$, then the next visited point $y$ has distortion at least $d + 1$ and $\HD{x}{y} \le d\log^2(n))$ with probability at least $1 - n^{-\varepsilon / 8}$. We call jumps in which this happens \emph{good} and jumps for which this does not hold \emph{bad}.

  We now show that our process reaches a point of distortion $d_1 \coloneqq \min\{\hat{d}/2,  n^{\varepsilon/16}\}$ and to this end, we consider the next $m = d_1$ points visited after $x_0$, which we denote by $x_1, \ldots, x_m$ where $x_i$ has distortion $d_i$. We show that there is a point with distortion $\ge d_1$ among these visited points w.h.p. To this end, denote by $\mathcal{E}_i$ the event that the jump from $x_{i-1}$ to $x_i$ is good or that one of the points $x_0, \ldots, x_{i-1}$ already has distortion $\ge d_1$. \Cref{lem:probincrease} tells us that (conditional on $\goodevent$) the probability of $\mathcal{E}_i$ is least $1 - n^{-\varepsilon / 8}$ but only if $\ZM(x_{i-1}) \le n^{1- \varepsilon}$. However, if $\mathcal{E}_1 \cap \ldots \cap \mathcal{E}_{i-1}$ hold, then we know that we either already reached a point of distortion $\ge d_1$ or \begin{align*}
   \ZM(x_{i-1}) &\le n^{1 - \varepsilon}/2 + i d  \log^2(n)  \\
   &\le n^{1 - \varepsilon}/2 + m n^{\varepsilon/16} \log^2(n) \le n^{1 - \varepsilon}
  \end{align*} by the second statement in \Cref{lem:probincrease} and because $m, d \le n^{\varepsilon/16}$. By the same reasoning, we get that \begin{align*}
    \ZM(x_{i-1}) &\ge n^{1 - 3\varepsilon} - m n^{\varepsilon/16} \log^2(n) \ge n^{1 - 4\varepsilon},
  \end{align*} so if $\bigcap_{j=1}^{i-1} \mathcal{E}_j$ occurs, we either already found a point $x$ of distortion $\ge d_1$ and $\ZM(x) \in [n^{1-4\varepsilon}, n^{1- \varepsilon}]$ or we know that the prerequisites of \Cref{lem:probincrease} are met. Hence, we conclude that \begin{align*}
    \Pr{ \mathcal{E}_i \mid \left( \bigcap_{j = 1}^{i-1} \mathcal{E}_j \right)  } \ge 1 - n^{-\varepsilon/8}
  \end{align*}
  and accordingly\footnote{It might be tempting to just use a union bound here. However, in our situation, we can only bound the probability of $\mathcal{E}_i$ if we know that we start at a point with $\ZeroMax$-value at most $n^{1 - \varepsilon}$, so a classical union bound is not applicable}, \begin{align*}
    \Pr{ \bigcap_{i = 1}^m \mathcal{E}_i} &= \prod_{i=1}^{m} \Pr{\mathcal{E}_i \mid \left( \bigcap_{j = 1}^{i-1} \mathcal{E}_j \right) }\\
    &\ge  \left( 1- n^{-\varepsilon/8} \right)^{n^{\varepsilon/16}}\\
    &\ge 1 - n^{-\varepsilon/8}n^{\varepsilon/16} \ge 1 - n^{-\varepsilon/16}
  \end{align*} where in the penultimate step we used Bernoulli's inequality. Accordingly, $\bigcap_{i = 1}^m \mathcal{E}_i$ occurs with high probability when starting at $x_0$, which -- by the definition of $\mathcal{E}_i$ -- implies that a point of distortion at least $d_1 = \min\{\hat{d}/2, n^{\varepsilon/16}\}$ is reached.
\end{proof}

\subsection{Lower Bounds for the Run Time of the \olea}

We continue by using the fact that the \olea likely finds a point of very large distortion to derive explicit lower bounds for its runtime that -- in its most general form -- depend explicitly on the CDF of our distribution $\mathcal{D}$. We then sketch the implications of this bound for some common distributions below.

\olealowerbound*
\begin{proof}
    By \Cref{thm:oleadistortionisreached}, the following event occurs with high probability for some sufficiently small constant $\varepsilon > 0$: Either $\timehorizon = \timehorizonvalue$ iterations pass while $\ZM(x) \ge n^{1 - \varepsilon} / 4$ or we find a point $x$ of distortion at least $d =\min\{ \hat{d}/2, n^{\varepsilon / 16}\}$ before finding the first point with fewer than $n^{1-4\varepsilon}$ zeros.    
    
    In the first case, our statement follows. For the second case -- if we find a point of distortion $d$ -- it is easy to see that the number of samples required until the first point with distortion at least $d$ is found dominates a geometric random variable with success probability $p \Pr{D \ge d}$. Thus the probability of finding not finding a such point after $1/ (g(n) p \Pr{D \ge d})$ samples is at least \begin{align*}
        \left( 1 - p \Pr{D \ge d} \right)^{\frac{1}{g(n) p \Pr{D \ge d}}} \ge 1 - 1/g(n) = 1 - o(1)
    \end{align*} as desired.
\end{proof} 

We show the implications of the above theorem for some common distributions.

\paragraph{Exponential Distribution} If $\mathcal{D}$ is an exponential distribution, i.e., if \begin{align*}
  \Pr{D \ge d} = \exp(-\varrho d)
\end{align*} for some $\varrho > 0$, we obtain a stretched exponential lower bound. 

\expbound* \begin{proof}
  \Cref{thm:olealowerbound} immediately implies that w.h.p., \begin{align*}
    T \ge \min \left\{ \timehorizonvalue, \exp( (1 - o(1)) \varrho n^{\varepsilon/16}) \right\} = \exp(n^{\Omega(1)}).
  \end{align*} 
\end{proof}

\paragraph{Gaussian Distribution} We further consider the case where $\mathcal{D}$ is a Gaussian distribution with density \begin{align*}
    \rho(x) = \frac{1}{s\sqrt{2\pi}} \exp\left(-\frac{1}{2}\left( \frac{x}{s}\right)^2\right)
\end{align*} for $x \in [0,\infty)$. It is not immediately obvious that this distribution fulfills \Cref{ass:distribution} up to a certain point, but we show in the following that \Cref{ass:distribution} is in fact met for $d = o(\log(n))$. To this end, we bound \begin{align*}
    \frac{\Pr{D \ge z}}{\Pr{D \ge z + 1}} &= \frac{\int_{z}^{\infty} \rho(x) \text{d}x}{\int_{z+1}^{\infty} \rho(x) \text{d}x}\\
    &= \frac{\int_{z}^{z+1} \rho(x) \text{d}x + \int_{z+1}^{\infty} \rho(x) \text{d}x}{\int_{z+1}^{\infty} \rho(x) \text{d}x}\\
    &= 1 + \frac{\int_{z}^{z+1} \rho(x) \text{d}x}{\int_{z+1}^{\infty} \rho(x) \text{d}x}\\
    &\le 1 + \frac{\int_{z}^{z+1} \rho(x) \text{d}x}{\int_{z+1}^{z+2} \rho(x) \text{d}x}\\
    &\le 1 + \frac{ \rho(z) }{\rho(z+2) }
\end{align*} because \begin{align*}
    (a - b) f(b) \le \int_a^b f(x) \text{d}x \le (a - b)f(a)
\end{align*} for any monotonically decreasing function $f$. Plugging in the definition of $\rho(x)$ yields.
\begin{align*}
    \frac{\Pr{D \ge z}}{\Pr{D \ge z + 1}} &\le 1 + \frac{\exp\left( -\frac{1}{2} \left( \frac{z}{s} \right)^2 \right)}{\exp\left( -\frac{1}{2} \left( \frac{z+2}{s} \right)^2 \right)}\\
    &= 1 + \exp\left( - \frac{1}{2s^2} \left( z^2 - (z+2)^2 \right) \right)\\
    &= 1 + \exp\left(\frac{4z + 4}{2s^2} \right), 
\end{align*} which is $n^{o(1)}$ if $z = o(\log(n))$. Thus, \Cref{ass:distribution} is met for $\hat{d} = \frac{\log(n)}{\log\log(n)}$ and we obtain the following corollary.

\gaussbound*
\begin{proof}
    As shown above, the absolute value of a Gaussian distribution fulfills \Cref{ass:distribution} for $\hat{d} = \frac{\log(n)}{\log\log(n)}$. Thus, \Cref{thm:olealowerbound} yields a lower bound of \begin{align*}
        T \ge \min \left\{ \timehorizonvalue, \left( g(n) p \Pr{D \ge \hat{d}/2} \right)^{-1} \right\}
    \end{align*} w.h.p. By the definition of the Gaussian distribution, we have \begin{align*}
        \Pr{D \ge \hat{d}/2} &= \frac{1}{s \sqrt{2\pi}} \int_{\hat{d}/2}^\infty \exp\left(-\frac{1}{2}\left( \frac{x}{s}\right)^2\right) \text{d}x\\
        &\le \frac{1}{s \sqrt{2\pi}} \int_{\hat{d}/2}^\infty \frac{x}{s^2} \exp\left(-\frac{1}{2}\left( \frac{x}{s}\right)^2\right) \text{d}x\\
        &= \frac{1}{s \sqrt{2\pi}} \left[ - \exp\left(-\frac{1}{2}\left( \frac{x}{s}\right)^2\right) \right]_{\hat{d}/2}^{\infty}\\
        &= \frac{1}{s \sqrt{2\pi}} \exp\left(-\frac{1}{2}\left( \frac{\hat{d}}{2s}\right)^2\right) \\
        &= \exp\left(-\Omega\left( \frac{\log(n)}{\log\log(n)}\right)^2\right) = n^{-\omega(1)},
    \end{align*} where the second step holds for sufficiently large $n$ since $\hat{d}/2 = \omega(1)$.
\end{proof}

\paragraph{Pareto Distribution} If $\mathcal{D}$ is a Pareto distribution, i.e., if \begin{align*}
    \Pr{D \ge x} = \left( \frac{x}{x_0} \right)^{1 - \tau}
\end{align*} for $x \in [x_0, \infty)$ and $\tau \ge 2$, we show that the run time is polynomial with exponent dependent on $\tau$. We start by verifying that \Cref{ass:distribution} is met. To this end, note that
\begin{align*}
    \frac{\Pr{D \ge z}}{\Pr{D \ge z + 1}} &= \left( \frac{z+1}{z} \right)^{\tau - 1}\\
    &\le \left( 1 + \frac{1}{x_0} \right)^{\tau - 1} = \bigO{1}.
\end{align*} This immediately implies that
\paretobound*
\begin{proof}
    If $g(n) = \omega(1)$, then $x_0^{1-\tau}g(n) = \omega(1)$ as well. Hence, \Cref{thm:olealowerbound} implies a lower bound of \begin{align*}
        T &\ge \min \left\{ \timehorizonvalue, \left( x_0^{1-\tau}g(n) p \Pr{D \ge n^{\varepsilon/16}} \right)^{-1} \right\}\\ 
        &\ge \frac{n^{(\tau - 1)\varepsilon/16}}{g(n)p} 
    \end{align*} 
    for every $g(n) = \omega(1)$, w.h.p.
\end{proof}

\section{Analysis of the $(1, \lambda)$-EA}\label{sec:oclea}

In this section, we show that the \oclea finds the optimum in $\Theta(n \log(n))$ steps w.h.p. Our analysis is essentially the same as in \cite[Section 4]{Jorritsma_Lengler_Sudholt_2023} with a minor technical modification. We give an overview of the analysis and the novel parts without repeating all the proofs. We use this section to prove the following theorem.
\ocleaupperbound*

Like in \cite[Section 4]{Jorritsma_Lengler_Sudholt_2023}, the analysis is based on a drift argument on an auxiliary dynamic version of our fitness function, which we call $\DyDisOmD$ and in which we reveal distorted and clean points gradually, and where distorted points can become non-distorted (but not vice versa). We adopt the terminology and notation from~\cite{Jorritsma_Lengler_Sudholt_2023} and call non-distorted points ``clean''. Formally, we let $x_t$ be the parent individual in the $t$-th iteration, and for $s = \lambda t + j$ for $j \in [\lambda]$, we let $y_s$ be the $j$-th offspring created in the $t$-th iteration. Furthermore, we let $C_s$ be the set of clean points after creating $y_s$, and iteratively define $C_0 = \{x_0\}$ and \begin{align*}
  C_{s} \coloneqq \begin{cases}
    C_{s - 1} \cup \{y_s\} & \text{w.p. } p,\text{ if } y_s \neq x_t \\ 
    C_{s-1} & \text{otherwise.}
  \end{cases}
\end{align*}
Then, \begin{align*}
  \DyDisOmD(x) = \begin{cases}
    \OM(x) & \text{if } x \in C_s\\
    \OM(x) + D & \text{otherwise},
  \end{cases}
\end{align*} where $D \sim \mathcal{D}$ is a sample from the distribution $\mathcal{D}$. To define our potential function as in \cite{Jorritsma_Lengler_Sudholt_2023}, we abbreviate $\ZM(x) = k$ and let $Y_1, \ldots, Y_\lambda \sim \text{Bin}(k, 1/n)$. We further define $Y^* \coloneqq \max\{ Y_1, \ldots, Y_\lambda \}$. Our potential function is \begin{align*}
  P(x) \coloneqq \ZM(x) + \mathds{1}(x \notin C_s) \frac{\delta}{\lambda p} \Expected{Y^*(\ZM(x))}.
\end{align*} 

The intuition behind this potential function is as follows. For a normal run of the \oclea on $\OneMax$, it is not hard to see that the drift of the potential $\ZM(x)$ is $\Omega(\Expected{Y^*(\ZM(x))})$ (because our progress is essentially governed by the maximum number of zero to one flips in a single offspring). By adding a penalty of $\frac{\delta}{\lambda p} \Expected{Y^*(\ZM(x))}$ to distorted points, we decrease said positive drift by a value in the same order of magnitude, since a distorted point is sampled with probability at most $\lambda p$, so the overall drift on clean points remains $\Omega(\Expected{Y^*(\ZM(x))})$ if we choose a suitable constant $\delta$. If the algorithm starts at a distorted point, then we also have a positive drift of the same order because we likely generate a set of offspring that does not include a clone of the current individual and no distorted point, while also incurring only a small increase in $\ZM(x)$. Hence, our potential decreases by approximately $\frac{q}{\lambda p} \Expected{Y^*(\ZM(x))} \ge \Expected{Y^*(\ZM(x))}$ due to our assumption that $\frac{q}{\lambda p} = \omega(1)$. All this was formally calculated in~\cite{Jorritsma_Lengler_Sudholt_2023}, and most cases go through without modifications, as we show below. 

%

The only difference to the analysis in \cite{Jorritsma_Lengler_Sudholt_2023} is that the negative drift incurred by starting at a distorted point increases slightly because we can now increase the potential function even if we clone our current individual by finding a point of larger distortion and larger $\ZeroMax$-value, which was impossible in~\cite{Jorritsma_Lengler_Sudholt_2023} where all points had the same distortion. However, this increase is still small enough to keep the overall (asymptotic) drift unchanged. We show this formally by computing explicit bounds on the drift \begin{align*}
    \Delta(x) \coloneqq \Expected{ P(x_t) - P(x_{t+1}) \mid x_t = x}
\end{align*} in the following lemma.

\begin{lemma}\label{lem:driftDyDisOmD}
    Consider a run of the \oclea on $\DyDisOmD$ under \Cref{ass:distribution} and for a sufficiently small constant $\delta > 0$. Then while $\ZM(x) \le k^*$, \begin{align*}
        \Delta(x) = \Omega(\Expected{Y^*(\ZM(x))}).
    \end{align*}
\end{lemma}

For the proof we further need the following auxiliary lemma from \cite{Jorritsma_Lengler_Sudholt_2023}.
\begin{lemma}\label{lem:auxpotential}
    Under \Cref{ass:params} and for $k \ge k^*$, we have \begin{enumerate}
        \item[(1)] $\Expected{Y^*(k + \log(n))} = (1 + o(1)) \Expected{Y^*(k)}$.
        \item[(2)] $\lambda p \log(n) = o(\Expected{Y^*(k)}).$
        \item[(3)] $q \log(n) = o(\Expected{Y^*(k)}).$
        \item[(4)] $1/n = o(\Expected{Y^*(k)}).$
        \item[(5)] for all $k \le n/\lambda$, we have $\Expected{Y^*(k)} = \Omega(P(x) \frac{\lambda}{n})$.
    \end{enumerate}
\end{lemma}

\begin{proof}[Proof of \Cref{lem:driftDyDisOmD}]
    As in \cite{Jorritsma_Lengler_Sudholt_2023}, we split the drift into a positive and negative component \begin{align*}
        \Delta^+(x) &\coloneqq \Expected{\max\{ P(x_t) - P(x_{t+1}), 0 \} \mid x_t = x  }\\
        \Delta^-(x) &\coloneqq \Expected{\max\{ P(x_{t+1}) - P(x_{t}), 0 \} \mid x_t = x }
    \end{align*}
    such that $\Delta(x) = \Delta^+(x) - \Delta^-(x)$. We further abbreviate $r^+:=\max\{r, 0\}$. Now, we distinguish the case that $x$ is clean and distorted, respectively, and derive bounds on the positive and negative drift in each case. The only case that is significantly different from the case distinction in \cite{Jorritsma_Lengler_Sudholt_2023} is the case of backward progress for distorted points. For all other cases, we therefore only give a rough proof sketch and refer the interested reader to \cite{Jorritsma_Lengler_Sudholt_2023} for details.
    \paragraph{Forward progress, clean points.} Here, it suffices to consider the event that all $\lambda$ offspring we create are clean. (I.e., we lower bound the forward progress of all other cases by zero.)  Clearly, this happens with probability at least $1 - \lambda p = 1 - o(1)$. Then, we can simply compute the expected maximal number of bits that flip from 0 to 1 in all offspring, which is $\Expected{Y^*(\ZM(x))}$. If we additionally consider the event that no bit flips from 1 to 0 in the offspring with the maximum number of 0-to-1 flips, our expected potential decrease is $\Omega(\Expected{Y^*(\ZM(x))})$. As this occurs with probability at least $1/e$, we get $\Delta^+(x) \ge \Omega(\Expected{Y^*(\ZM(x))})$ in total.
    \paragraph{Forward progress, distorted points.} In this case, we consider the event $\mathcal{E}_1$ that all offspring are clean, the event $\mathcal{E}_2$ that there is no clone of $x$ among the offspring, and the event $\mathcal{E}_3$ that there is an offspring that flips at most one 1-bit to a 0. These events together imply that we accept a clean point that increases the number of zeros by at most $1$ and thus the potential decreases by at least $(\delta/(\lambda p) \Expected{Y^*(\ZM(x))}) - 2$. Note that these events and bounds hold regardless of whether we use constant distortions or not. In total, as shown in \cite[Proof of Lemma 4.2]{Jorritsma_Lengler_Sudholt_2023}, this argument yields a drift of \begin{align}\label{eq:positivedriftdistortedpoints}
        \Delta^+(x) \ge  \frac{\delta q}{4\lambda p}\Expected{Y^*(\ZM(x))}
    \end{align}
    \paragraph{Backward progress, clean points.} Here, to upper bound the negative drift, we consider the event $\mathcal{E}_5$ that all offspring differ in at most $\log(n)$ bits from $x$. It is easy to show that \begin{align*}
        \Expected{ (P(x_{t+1}) - P(x_t))^+ \mathds{1}(\neg \mathcal{E}_5) } = o(1/n)
    \end{align*} since $\neg \mathcal{E}_5$ is very unlikely. However, if $\mathcal{E}_5$ occurs, then we know that the $\ZeroMax$-value of the next accepted point increases by at most $\log(n)$. If there is additionally a distorted offspring -- we denote this event by $\mathcal{E}_4$ -- then the potential increases by at most $\log(n) + \delta/(\lambda p) \Expected{Y^*(\ZM(x) + \log(n))}$ and hence  \begin{align*}
        &\Expected{ (P(x_{t+1}) - P(x_t))^+\mathds{1}(\mathcal{E}_5)\mathds{1}(\mathcal{E}_4) } \\
        & \hspace{3cm} \le \log(n) \lambda p + 2\delta \Expected{Y^*(\ZM(x))}.
    \end{align*}
    On the other hand, if $\neg \mathcal{E}_4$ occurs, i.e., if there is no distorted offspring, then the potential can only increase if we do not create a clone of $x$. Denote this event (that is the event of not cloning $x$) by $\mathcal{E}_2$ and recall that $\Pr{\mathcal{E}_2}= (1+o(1))q$. As $\neg \mathcal{E}_4 \cap \mathcal{E}_5$ implies that there is no distorted offspring and no offspring with more than $\log(n)$ bit flips, the potential increases by at most $\log(n)$ if $\neg \mathcal{E}_4 \cap \mathcal{E}_5 \cap \mathcal{E}_2$ occurs. Hence, \begin{align*}
        \Expected{ (P(x_{t+1}) - P(x_t))^+\mathds{1}(\mathcal{E}_5)\mathds{1}(\neg \mathcal{E}_4)\mathds{1}(\mathcal{E}_2) } \le (1+o(1))q \log(n).
    \end{align*}
    In total, this yields 
    \begin{align*}
        &\Delta^-(x) \le 2\delta\Expected{Y^*(\ZM(x))} \\
        &\hspace{2cm}+ \bigO{\lambda p \log(n) }+ \bigO{q \log(n)} + o(1/n).
    \end{align*} Due to our parameter setup and $\ZM(x) \ge k^*$, this term is $(1 + o(1))2\delta \Expected{Y^*(\ZM(x))}$ by \Cref{lem:auxpotential} (we refer to \cite{Jorritsma_Lengler_Sudholt_2023}) for the computation details), which is smaller than our lower bound on $\Delta^+(x)$ for clean points if we choose $\delta$ small enough. Hence, for clean points, we conclude that \begin{align*}
        \Delta(x) \ge \Delta^+(x) - \Delta^-(x) = \Omega(\Expected{Y^*(\ZM(x)})).
    \end{align*}

    \paragraph{Backward progress, distorted points.} In this case our analysis will slightly divert from that of \cite{Jorritsma_Lengler_Sudholt_2023}. To this end, recall that $\mathcal{E}_2$ is the event that there is no clone of $x$ among the offspring, and that $\mathcal{E}_5$ is the event that there is no offspring with more than $\log(n)$ bit flips. If $\mathcal{E}_2 \cap \mathcal{E}_5$ occurs, the reasoning is just like in  \cite{Jorritsma_Lengler_Sudholt_2023}: In the worst case, the next accepted point is distorted and its $\ZeroMax$-value changes by at most $\log(n)$. Thus, \begin{align*}
        &\Expected{(P(x_{t+1}) - P(x_t))^+ \mathds{1}(\mathcal{E}_2)\mathds{1}(\mathcal{E}_5)} \\
        &\hspace{.5cm}\le \Pr{\mathcal{E}_2} (  \log(n) \\ 
        &\hspace{.8cm}+  \delta / (\lambda p) \left( \Expected{Y^*(\ZM(x) + \log(n))} - \Expected{Y^*(\ZM(x))} \right)) \\
        &\hspace{.5cm}\le 2q \left( \log(n) + \frac{\delta}{\lambda p} o\left(\Expected{Y^*(\ZM(x))}\right) \right)
    \end{align*} because $\Pr{\mathcal{E}_2} \le 2q$ and because of \Cref{lem:auxpotential} item (1). The new part in our analysis is the case where there is a clone among the offspring. In our case, we can increase the potential in this case as well, namely, when we find a point of larger distortion that has a greater $\ZeroMax$-value. However, this is only possible if we sample a distorted offspring, which happens only with probability at most $\lambda p$. So similarly as above, we have \begin{align*}
        &\Expected{(P(x_{t+1}) - P(x_t))^+ \mathds{1}(\neg\mathcal{E}_2)\mathds{1}(\mathcal{E}_5)}
        \\&\hspace{1cm}\le \lambda p \left( \log(n) + \frac{\delta}{\lambda p} o\left(\Expected{Y^*(\ZM(x))}\right) \right).
    \end{align*} In total, \begin{align*}
        \Delta^{-}(x) &= \Expected{(P(x_{t+1}) - P(x_t))^+ \mathds{1}(\mathcal{E}_2)\mathds{1}(\mathcal{E}_5)}\\
        & \hspace{1cm} + \Expected{(P(x_{t+1}) - P(x_t))^+ \mathds{1}(\neg\mathcal{E}_2)\mathds{1}(\mathcal{E}_5)}\\
        & \hspace{1cm} + \Expected{(P(x_{t+1}) - P(x_t))^+ \mathds{1}(\neg\mathcal{E}_5)}\\
        &\le (2q + \lambda p) \left( \log(n) + \frac{\delta}{\lambda p} o\left(\Expected{Y^*(\ZM(x))}\right) \right) + o(1/n).
    \end{align*} By items (2) - (4) in \Cref{lem:auxpotential}, this implies that \begin{align*}
        \Delta^{-}(x) &\le \left( \frac{(2q + \lambda p) \delta}{\lambda p} + 1 \right) o\left(\Expected{Y^*(\ZM(x))}\right) \\
        &= \left(\frac{2q\delta}{\lambda p} + \delta + 1 \right) o\left(\Expected{Y^*(\ZM(x))}\right).
    \end{align*} Using the bound in \Cref{eq:positivedriftdistortedpoints}, we conclude that for distorted points, we have a total drift of \begin{align*}
        \Delta(x) &\ge \frac{\delta q}{\lambda p} \left( \frac{1}{4}\Expected{Y^*(\ZM(x))} - o\left(\Expected{Y^*(\ZM(x))}\right) \right) \\
        & \hspace{1cm} - \left(\delta + 1 \right) o\left(\Expected{Y^*(\ZM(x))}\right) \\
        &= \Omega\left( \frac{q}{\lambda p} \Expected{Y^*(\ZM(x))} \right) = \Omega\left( \Expected{Y^*(\ZM(x))} \right)
    \end{align*} because $q = \omega(p\lambda)$.
\end{proof}

Overall, we use this to prove \Cref{thm:ocleabound}. \begin{proof}[Proof of \Cref{thm:ocleabound}]
    The proof is identical to the proof of \cite[Theorem 1.1]{Jorritsma_Lengler_Sudholt_2023} using \Cref{lem:driftDyDisOmD} instead of \cite[Lemma 4.3]{Jorritsma_Lengler_Sudholt_2023}.
\end{proof}

\section{Experiments}\label{sec:experiments}
We confirm our theoretical results empirically in multiple settings\footnote{The code is available at \url{https://github.com/OliverSieberling/DistortedOneMax}}. First of all, \Cref{fig:oleaexp} shows how the total fitness, \OneMax-fitness, and the distortion evolve over a single run on \DisOMpar with an exponential distribution. We can clearly see that the \olea quickly reaches a local optimum and then tends to increase distortion further while decreasing $\OneMax$-fitness. Furthermore, progress becomes extremely slow, note the logarithmic $x$-axis. After one billion generations the \olea is still moderately far away from the optimum while not having accepted a single offspring in the prior $975$ million generations. On the other hand, the \oclea escapes the local optima it encounters quickly and reaches fitness $n - k^*$ in less than $400$ generations.

\begin{figure}
    \centering
    \resizebox{\linewidth}{!}{\input{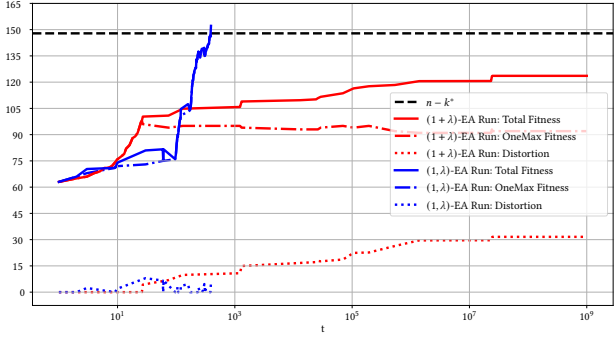}}
    \caption{\normalfont Total fitness, \OneMax-fitness, and distortion of one run of the \olea and the \oclea on \DisOMpar with $\mathcal{D}$ being an exponential distribution with rate parameter $0.4$. We use $n = 150$, $\lambda = 8$, $p = 0.0245$ and $k^* = 2.12$ with a cutoff of $10^9$ generations. The x-axis is scaled logarithmically.}\label{fig:oleaexp}
\end{figure}

This behavior is predominant even for smaller problem sizes. In \Cref{fig:MedianExpPlusVsComma} we plot the median number of generations required by the \olea and the \oclea to optimize \DisOMpar. For distortions sampled from an exponential distribution we observe that already for $n=90$ a majority of the elitist runs exceed the cutoff of one million generations, while the \oclea optimizes efficiently in every single run. This illustrates the exponential slowdown proven in \Cref{cor:expbound}. We remark that for small problem sizes the sampled distortions have an significant impact on the total fitness. Hence, the \olea can find a point of sufficient fitness without being close to the optimum in terms of \OneMax fitness simply by increasing distortion. For distortions sampled from a uniform distribution, which also satisfies Assumption~\ref{ass:distribution} but is much less ``heavy-tailed'', we see a smaller slowdown. This indicates that the governing factor of the run time of the \olea on \DisOMpar is indeed controlled by the ``heavy-tailedness'' of $\mathcal{D}$ as predicted by \Cref{thm:olealowerbound}.

\begin{figure}
    \centering
    \resizebox{\linewidth}{!}{\input{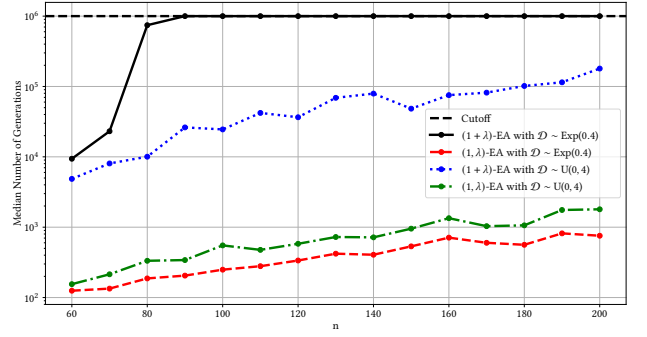}}
    \caption{\normalfont Number of generations required by the \olea and the \oclea to optimize \DisOMpar with $\mathcal{D}$ being an exponential distribution with rate parameter $0.4$ and a uniform distribution from $0$ to $4$. We take the median over $49$ runs. We set $\lambda = \lfloor 1.5 \log{n} \rceil$, $p = 0.3n^{-0.5}$ and $k^* = n^{0.15}$ with a cutoff of $10^6$ generations. The y-axis is scaled logarithmically.} 
    \label{fig:MedianExpPlusVsComma}
\end{figure}

In \Cref{fig:UniformDifferentP}, we investigate this question further and conclude that the tern $(p \Pr{D \ge d})^{-1}$ predicted by \Cref{thm:olealowerbound} as a lower bound for the run time of the \olea can be empirically observed. For this purpose, we compare the performance of the \olea on \DisOMpar with distortions sampled from different truncated exponential distributions for different $p$, and we plot the run time, rescaled by $(p \Pr{D \ge d})^{-1}$ (where $d$ is the cutoff) in \Cref{fig:UniformDifferentP}. We observe that the curves are grouped closely together as we would expect, although the run times seem to grow slightly faster than by a factor of $1/p$. The same figure also shows that the optimization time increases significantly when truncating the distribution $\mathcal{D}$ at higher values. This slowdown seems to be roughly $1/Pr[D \ge d]$ as we would expect.

\begin{figure}
    \centering
    \resizebox{\linewidth}{!}{\input{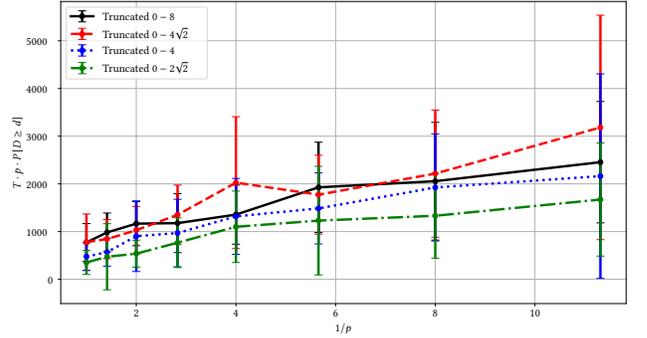}}
    \caption{\normalfont Normalized number of generations required by the \olea to optimize \DisOMpar for different distortion probabilities $p$. The distortions are sampled from an exponential distribution with rate parameter $0.4$ truncated at different cutoffs $d$. We set $n=300$, $\lambda = 9$, $k^* = 2.35$ and average over $49$ runs. Note that the $y$-axis shows the (averaged) run time re-scaled by a factor of $(p\Pr{D \ge d})^{-1}$. This gives evidence that the run time is indeed proportional to $(p\Pr{D \ge d})^{-1}$ as predicted by \Cref{thm:olealowerbound}.} 
    \label{fig:UniformDifferentP}
\end{figure}

\section{Conclusion}

We have shown that plus strategies are prone to get stuck in randomly placed local optima of random height, modeled by the benchmark function \DisOMpar. We gave a sufficient condition (\Cref{ass:distribution}) on the distribution $\mathcal{D}$ ensuring that -- sufficiently close to the optimum -- once we find a distorted point, things only get worse and the \olea accepts points of higher and higher distortion, which leads to super-polynomial run times for multiple natural choices of $\mathcal{D}$ including the exponential and the Gaussian distribution. This holds even if the local optima are planted relatively sparsely, for example $p = n^{-1+\eps}$. Our illustrating example is the exponential distribution, for which we obtain a stretched exponential lower bound of $\exp(n^{\Omega(1)})$. Our results hence indicate that plus strategies on natural ``rugged'' fitness landscapes are unsuitable for practical applications. 

On the other hand, for some parameter choices for which plus strategies need super-polynomial time, we showed that comma strategies optimize \DisOMpar asymptotically in the same time as \OneMax, i.e., in only $\Theta(n \log(n))$ fitness evaluations. Our analysis here relies on the assumption that local optima are planted sufficiently sparsely, i.e., that $p$ is sufficiently small. While this is exactly the same assumption as in previous work \cite{Jorritsma_Lengler_Sudholt_2023}, it is likely not needed to ensure efficiency of the \oclea, as indicated by our experimental evaluation (\Cref{fig:oleaexp}) which takes place in a setting where the assumption $p = o(k^*/n)$ is not met. We leave it as an open question for future work to extend the theoretical study of the \oclea on \DisOMpar for a less restrictive parameter regime. 

While we studied a distorted version of \onemax, any other benchmark function can be distorted in the same way. It would be interesting to study whether comma strategies remain efficient on distorted versions of other benchmarks, for example \leadingones, linear functions, or monotone functions. Finally, there are many other non-elitist selection strategies than comma selection, like tournament selection or, and there are situations in which those other strategies are preferable~\cite{dang2021non}. It would be important to know which other non-elitist selection strategies can also deal well with distortions.



\bibliographystyle{ACM-Reference-Format}
\bibliography{sample-base}


\begin{thebibliography}{14}


\ifx \showCODEN    \undefined \def \showCODEN     #1{\unskip}     \fi
\ifx \showDOI      \undefined \def \showDOI       #1{#1}\fi
\ifx \showISBNx    \undefined \def \showISBNx     #1{\unskip}     \fi
\ifx \showISBNxiii \undefined \def \showISBNxiii  #1{\unskip}     \fi
\ifx \showISSN     \undefined \def \showISSN      #1{\unskip}     \fi
\ifx \showLCCN     \undefined \def \showLCCN      #1{\unskip}     \fi
\ifx \shownote     \undefined \def \shownote      #1{#1}          \fi
\ifx \showarticletitle \undefined \def \showarticletitle #1{#1}   \fi
\ifx \showURL      \undefined \def \showURL       {\relax}        \fi
\providecommand\bibfield[2]{#2}
\providecommand\bibinfo[2]{#2}
\providecommand\natexlab[1]{#1}
\providecommand\showeprint[2][]{arXiv:#2}

\bibitem[\protect\citeauthoryear{Antipov, Doerr, and Yang}{Antipov et~al\mbox{.}}{2019}]%
        {Antipov_Doerr_Yang_2019}
\bibfield{author}{\bibinfo{person}{Denis Antipov}, \bibinfo{person}{Benjamin Doerr}, {and} \bibinfo{person}{Quentin Yang}.} \bibinfo{year}{2019}\natexlab{}.
\newblock \showarticletitle{The efficiency threshold for the offspring population size of the $(\mu, \lambda)$ {EA}}. In \bibinfo{booktitle}{\emph{Proceedings of the Genetic and Evolutionary Computation Conference}} \emph{(\bibinfo{series}{GECCO ’19})}. \bibinfo{publisher}{Association for Computing Machinery}, \bibinfo{address}{New York, NY, USA}, \bibinfo{pages}{1461–1469}.
\newblock
\showISBNx{978-1-4503-6111-8}
\urldef\tempurl%
\url{https://doi.org/10.1145/3321707.3321838}
\showDOI{\tempurl}


\bibitem[\protect\citeauthoryear{Auger, Fonseca, Friedrich, Lengler, and Gissler}{Auger et~al\mbox{.}}{2022}]%
        {dagstuhl22081theory}
\bibfield{author}{\bibinfo{person}{Anne Auger}, \bibinfo{person}{Carlos~M. Fonseca}, \bibinfo{person}{Tobias Friedrich}, \bibinfo{person}{Johannes Lengler}, {and} \bibinfo{person}{Armand Gissler}.} \bibinfo{year}{2022}\natexlab{}.
\newblock \showarticletitle{{Theory of Randomized Optimization Heuristics (Dagstuhl Seminar 22081)}}.
\newblock \bibinfo{journal}{\emph{Dagstuhl Reports}} \bibinfo{volume}{12}, \bibinfo{number}{2} (\bibinfo{year}{2022}), \bibinfo{pages}{87--102}.
\newblock
\showISSN{2192-5283}


\bibitem[\protect\citeauthoryear{Dang, Eremeev, and Lehre}{Dang et~al\mbox{.}}{2021}]%
        {dang2021non}
\bibfield{author}{\bibinfo{person}{Duc-Cuong Dang}, \bibinfo{person}{Anton Eremeev}, {and} \bibinfo{person}{Per~Kristian Lehre}.} \bibinfo{year}{2021}\natexlab{}.
\newblock \showarticletitle{Non-elitist evolutionary algorithms excel in fitness landscapes with sparse deceptive regions and dense valleys}. In \bibinfo{booktitle}{\emph{Genetic and Evolutionary Computation Conference (GECCO 2021)}}. \bibinfo{pages}{1133--1141}.
\newblock


\bibitem[\protect\citeauthoryear{Doerr}{Doerr}{2022}]%
        {doerr22does}
\bibfield{author}{\bibinfo{person}{Benjamin Doerr}.} \bibinfo{year}{2022}\natexlab{}.
\newblock \showarticletitle{Does comma selection help to cope with local optima?}
\newblock \bibinfo{journal}{\emph{Algorithmica}}  \bibinfo{volume}{84} (\bibinfo{year}{2022}), \bibinfo{pages}{1659--1693}.
\newblock


\bibitem[\protect\citeauthoryear{Dubhashi and Panconesi}{Dubhashi and Panconesi}{2009}]%
        {Dubhashi_Panconesi_2009}
\bibfield{author}{\bibinfo{person}{Devdatt~P. Dubhashi} {and} \bibinfo{person}{Alessandro Panconesi}.} \bibinfo{year}{2009}\natexlab{}.
\newblock \bibinfo{booktitle}{\emph{Concentration of Measure for the Analysis of Randomized Algorithms}}.
\newblock \bibinfo{publisher}{Cambridge University Press}.
\newblock
\showISBNx{978-1-139-48099-4}
\newblock
\shownote{Google-Books-ID: UUohAwAAQBAJ}.


\bibitem[\protect\citeauthoryear{Friedrich, K{\"o}tzing, Krejca, and Sutton}{Friedrich et~al\mbox{.}}{2016a}]%
        {friedrich2016compact}
\bibfield{author}{\bibinfo{person}{Tobias Friedrich}, \bibinfo{person}{Timo K{\"o}tzing}, \bibinfo{person}{Martin~S Krejca}, {and} \bibinfo{person}{Andrew~M Sutton}.} \bibinfo{year}{2016}\natexlab{a}.
\newblock \showarticletitle{The compact genetic algorithm is efficient under extreme gaussian noise}.
\newblock \bibinfo{journal}{\emph{IEEE Transactions on Evolutionary Computation}} \bibinfo{volume}{21}, \bibinfo{number}{3} (\bibinfo{year}{2016}), \bibinfo{pages}{477--490}.
\newblock


\bibitem[\protect\citeauthoryear{Friedrich, K{\"o}tzing, Krejca, and Sutton}{Friedrich et~al\mbox{.}}{2016b}]%
        {friedrich2016graceful}
\bibfield{author}{\bibinfo{person}{Tobias Friedrich}, \bibinfo{person}{Timo K{\"o}tzing}, \bibinfo{person}{Martin~S Krejca}, {and} \bibinfo{person}{Andrew~M Sutton}.} \bibinfo{year}{2016}\natexlab{b}.
\newblock \showarticletitle{Graceful scaling on uniform versus steep-tailed noise}. In \bibinfo{booktitle}{\emph{International Conference on Parallel Problem Solving from Nature}}. Springer, \bibinfo{pages}{761--770}.
\newblock


\bibitem[\protect\citeauthoryear{Friedrich, K{\"{o}}tzing, Neumann, and Radhakrishnan}{Friedrich et~al\mbox{.}}{2022}]%
        {FriedrichKNR22}
\bibfield{author}{\bibinfo{person}{Tobias Friedrich}, \bibinfo{person}{Timo K{\"{o}}tzing}, \bibinfo{person}{Frank Neumann}, {and} \bibinfo{person}{Aishwarya Radhakrishnan}.} \bibinfo{year}{2022}\natexlab{}.
\newblock \showarticletitle{Theoretical Study of Optimizing Rugged Landscapes with the cGA}. In \bibinfo{booktitle}{\emph{Parallel Problem Solving from Nature ({PPSN}~2022)}}. \bibinfo{publisher}{Springer}, \bibinfo{pages}{586--599}.
\newblock


\bibitem[\protect\citeauthoryear{Hevia~Fajardo and Sudholt}{Hevia~Fajardo and Sudholt}{2021}]%
        {hevia2021a}
\bibfield{author}{\bibinfo{person}{Mario~Alejandro Hevia~Fajardo} {and} \bibinfo{person}{Dirk Sudholt}.} \bibinfo{year}{2021}\natexlab{}.
\newblock \showarticletitle{Self-Adjusting Offspring Population Sizes Outperform Fixed Parameters on the Cliff Function}. In \bibinfo{booktitle}{\emph{Foundations of Genetic Algorithms (FOGA 2021)}}, Vol.~\bibinfo{volume}{5}. \bibinfo{pages}{1--5}.
\newblock


\bibitem[\protect\citeauthoryear{Jorritsma, Lengler, and Sudholt}{Jorritsma et~al\mbox{.}}{2023}]%
        {Jorritsma_Lengler_Sudholt_2023}
\bibfield{author}{\bibinfo{person}{Joost Jorritsma}, \bibinfo{person}{Johannes Lengler}, {and} \bibinfo{person}{Dirk Sudholt}.} \bibinfo{year}{2023}\natexlab{}.
\newblock \showarticletitle{Comma Selection Outperforms Plus Selection on OneMax with Randomly Planted Optima}. In \bibinfo{booktitle}{\emph{Proceedings of the Genetic and Evolutionary Computation Conference}} \emph{(\bibinfo{series}{GECCO ’23})}. \bibinfo{publisher}{Association for Computing Machinery}, \bibinfo{address}{New York, NY, USA}, \bibinfo{pages}{1602–1610}.
\newblock
\showISBNx{9798400701191}
\urldef\tempurl%
\url{https://doi.org/10.1145/3583131.3590488}
\showDOI{\tempurl}


\bibitem[\protect\citeauthoryear{Lehre}{Lehre}{2011}]%
        {Lehre_2011}
\bibfield{author}{\bibinfo{person}{Per~Kristian Lehre}.} \bibinfo{year}{2011}\natexlab{}.
\newblock \showarticletitle{Fitness-levels for non-elitist populations}. In \bibinfo{booktitle}{\emph{Proceedings of the 13th annual conference on Genetic and evolutionary computation}}. \bibinfo{publisher}{ACM}, \bibinfo{address}{Dublin Ireland}, \bibinfo{pages}{2075–2082}.
\newblock
\showISBNx{978-1-4503-0557-0}
\urldef\tempurl%
\url{https://doi.org/10.1145/2001576.2001855}
\showDOI{\tempurl}


\bibitem[\protect\citeauthoryear{Lengler and Steger}{Lengler and Steger}{2018}]%
        {Lengler_Steger_2018}
\bibfield{author}{\bibinfo{person}{J. Lengler} {and} \bibinfo{person}{A. Steger}.} \bibinfo{year}{2018}\natexlab{}.
\newblock \showarticletitle{Drift Analysis and Evolutionary Algorithms Revisited}.
\newblock \bibinfo{journal}{\emph{Combinatorics, Probability and Computing}} \bibinfo{volume}{27}, \bibinfo{number}{4} (\bibinfo{date}{July} \bibinfo{year}{2018}), \bibinfo{pages}{643–666}.
\newblock
\showISSN{0963-5483, 1469-2163}
\urldef\tempurl%
\url{https://doi.org/10.1017/S0963548318000275}
\showDOI{\tempurl}


\bibitem[\protect\citeauthoryear{Prugel-Bennett, Rowe, and Shapiro}{Prugel-Bennett et~al\mbox{.}}{2015}]%
        {prugel2015run}
\bibfield{author}{\bibinfo{person}{Adam Prugel-Bennett}, \bibinfo{person}{Jonathan Rowe}, {and} \bibinfo{person}{Jonathan Shapiro}.} \bibinfo{year}{2015}\natexlab{}.
\newblock \showarticletitle{Run-time analysis of population-based evolutionary algorithm in noisy environments}. In \bibinfo{booktitle}{\emph{Proceedings of the 2015 ACM Conference on Foundations of Genetic Algorithms XIII}}. \bibinfo{pages}{69--75}.
\newblock


\bibitem[\protect\citeauthoryear{Rowe and Sudholt}{Rowe and Sudholt}{2014}]%
        {Rowe_Sudholt_2014}
\bibfield{author}{\bibinfo{person}{Jonathan~E. Rowe} {and} \bibinfo{person}{Dirk Sudholt}.} \bibinfo{year}{2014}\natexlab{}.
\newblock \showarticletitle{The choice of the offspring population size in the $(1,\lambda)$ evolutionary algorithm}.
\newblock \bibinfo{journal}{\emph{Theoretical Computer Science}}  \bibinfo{volume}{545} (\bibinfo{date}{Aug.} \bibinfo{year}{2014}), \bibinfo{pages}{20–38}.
\newblock
\showISSN{0304-3975}
\urldef\tempurl%
\url{https://doi.org/10.1016/j.tcs.2013.09.036}
\showDOI{\tempurl}


\end{thebibliography}

\appendix

\section{Appendix}

\end{document}